\documentclass[sigconf]{acmart}

\usepackage{algorithm}
\usepackage{algorithmic}
\usepackage{mathtools}
\usepackage{amsthm}
\usepackage{graphicx}
\usepackage{multirow}
\usepackage{subcaption}
\usepackage{enumitem}
\usepackage{booktabs}
\usepackage[table]{xcolor}
\usepackage{hyperref}
\usepackage{xcite}
\usepackage[capitalize,noabbrev]{cleveref}

\theoremstyle{plain}
\newtheorem{theorem}{Theorem}[section]

\theoremstyle{definition}

\newtheorem{assumption}[theorem]{Assumption}
\theoremstyle{remark}

\definecolor{Gray}{gray}{0.94}

\usepackage{xr-hyper}
\makeatletter
\newcommand*{\addFileDependency}[1]{%
  \typeout{(#1)}%
  \@addtofilelist{#1}%
  \IfFileExists{#1}{}{\typeout{No file #1.}}%
}

\makeatother
\newcommand*{\myexternaldocument}[1]{%
    \externaldocument{#1}%
    \addFileDependency{#1.tex}%
    \addFileDependency{#1.aux}%
}

\renewcommand{\footnote}[1]{}
\newcommand{\edit}[1]{\textcolor{black}{#1}} 

\myexternaldocument{supplementary}

\AtBeginDocument{%
  }

\copyrightyear{2026}
\acmYear{2026}
\setcopyright{cc}
\setcctype{by}
\acmConference[MM '26]{Proceedings of the 34th ACM International Conference on Multimedia}{November 10--14, 2026}{Rio de Janeiro, Brazil}
\acmBooktitle{Proceedings of the 34th ACM International Conference on Multimedia (MM '26), November 10--14, 2026, Rio de Janeiro, Brazil}
\acmDOI{10.1145/3767308.3836532}
\acmISBN{979-8-4007-2213-4/2026/11}

\acmSubmissionID{9332}

\begin{document}

\title{Learning What Not to Learn: Adversarial Disentangled Prompt Tuning for Robust Vision-Language Models}


\author{Yang Chen}
\orcid{0009-0002-5176-6690}
\affiliation{%
  \institution{Southern University of Science and Technology}
  \city{Shenzhen}
  \country{China}
}
\email{cheny2023@mail.sustech.edu.cn}

\author{Zhan Zhuang}
\orcid{0000-0003-0215-8728}
\affiliation{%
  \institution{City University of Hong Kong}
  \city{Hong Kong}
  \country{China}
}
\email{12250063@mail.sustech.edu.cn}

\author{Yanbin Wei}
\orcid{0000-0003-1301-2505}
\affiliation{%
  \institution{Hong Kong University of Science and Technology}
  \city{Hong Kong}
  \country{China}
}
\email{yanbin.ust@gmail.com}

\author{Zebin Chen}
\orcid{0009-0009-2233-1349}
\affiliation{%
  \institution{Southern University of Science and Technology}
  \city{Shenzhen}
  \country{China}
}
\email{12432660@mail.sustech.edu.cn}

\author{Hua Liu}
\correspondingauthor
\orcid{0000-0002-9613-8877}
\affiliation{%
  \institution{Southern University of Science and Technology}
  \city{Shenzhen}
  \country{China}
}
\email{liuh5@sustech.edu.cn}

\author{Yu Zhang}
\correspondingauthor
\orcid{0000-0003-1100-4835}
\affiliation{%
  \institution{Southern University of Science and Technology}
  \city{Shenzhen}
  \country{China}
}
\email{yu.zhang.ust@gmail.com}

\thanks{\ding{41}: Corresponding authors.}
\renewcommand{\shortauthors}{Chen et al.}

\begin{abstract}
  While adversarial prompt tuning can enhance robustness of vision-language models efficiently, we find that existing methods aggravate robust generalization overfitting on seen classes, leading to a rapid degradation in performance against adversarial examples of unseen classes as training progresses. We empirically identify that this degradation stems from the tendency of the model to learn pseudo-robust features (i.e., non-generalizable shortcuts). To mitigate this, we propose \textbf{ADAPT} (\textbf{A}dversarial \textbf{D}isent\textbf{A}ngled \textbf{P}rompt \textbf{T}uning), a robust prompt tuning framework following the philosophy of ``Learning What Not to Learn''. Specifically, ADAPT uses a dual-prompt mechanism with a target prompt and a pool of decoy prompts. During training, the decoy prompts are guided to entrap diverse pseudo-robust features, while the target prompt is constrained to be orthogonal to the decoys in the embedding space to learn robust features. By disentangling the robust features from the pseudo-robust features, ADAPT effectively prevents robust generalization overfitting. We further provide an analysis showing that the orthogonal loss bounds the effect of shifts in pseudo-robust features on unseen classes, yielding a testing error guarantee. Empirically, extensive experiments demonstrate that ADAPT substantially improves the robustness of the target prompt on unseen classes. The code is available at \url{https://github.com/cheny02/ADAPT-ACMMM2026}.
\end{abstract}

  


\maketitle

\section{Introduction}
Vision-language models (VLMs), such as CLIP \citep{clip} and its successors \citep{evaclip,evaclip18,longclip,align,openclip}, have achieved remarkable success in bridging the gap between visual and textual modalities. Through extensive pre-training on massive paired image-text data, these models have shown powerful zero-shot performance on unseen data. However, recent studies \citep{schlarmann2023, zhao2023} have shown that VLMs are extremely fragile to visual adversarial perturbations, often yielding confident but incorrect predictions when exposed to imperceptible noise \citep{explaining}. Such vulnerability raises concerns regarding the reliability of VLMs, prohibiting their adoption in safety-critical scenarios like autonomous driving \citep{wang2023does} and healthcare \citep{finlayson2019adversarial}, where a single adversarial failure could lead to unacceptable consequences.

\begin{figure}[!t]
\centering
\includegraphics[width=0.9\linewidth]{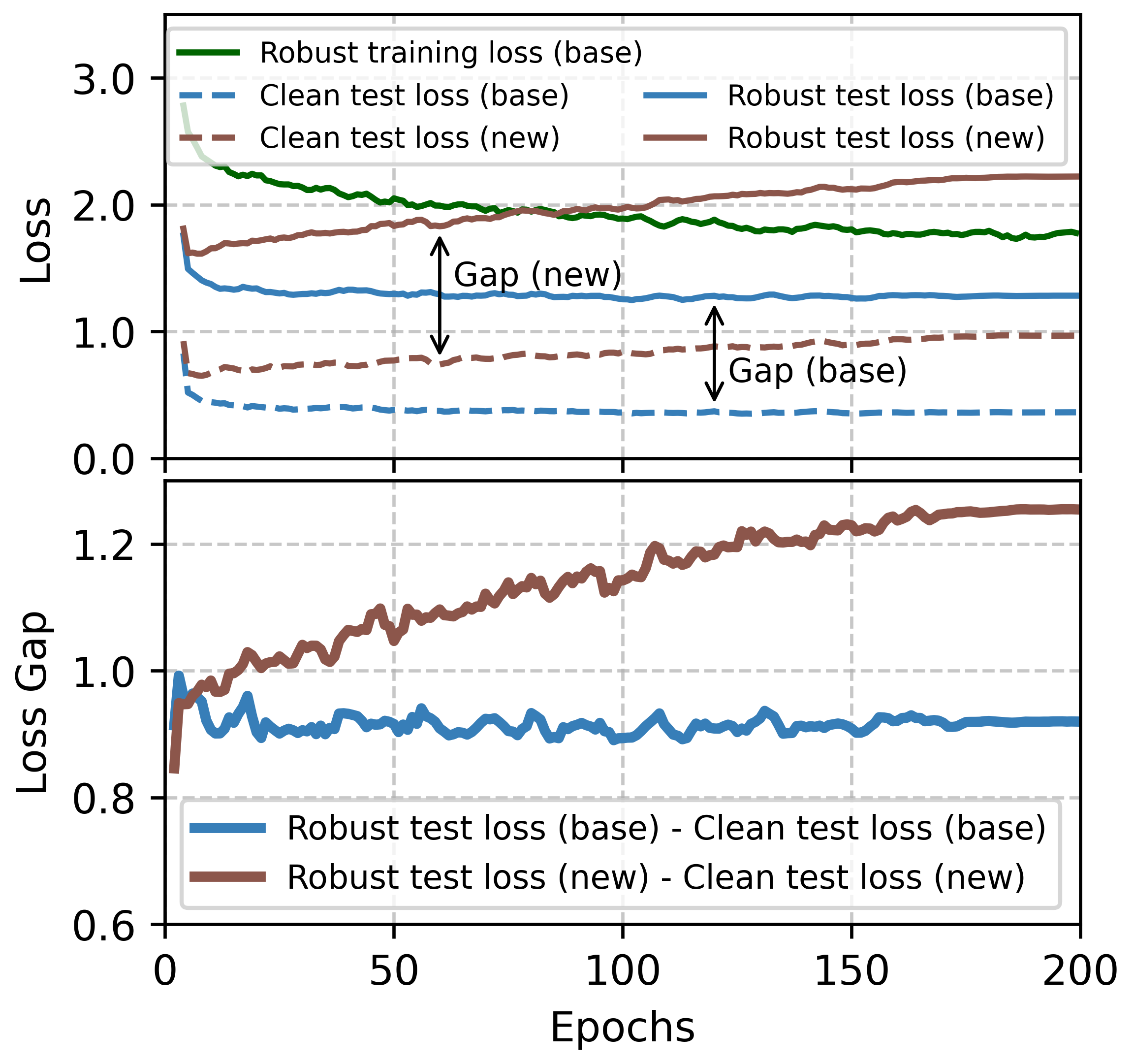}
\caption{Illustration of robust generalization overfitting. Adversarial prompt tuning on Caltech101 with CLIP ViT-B/32 under  PGD-$l_{\infty}$ attack ($\epsilon=4/255$).  Top: Robust loss remains stable on base (seen) classes but increases on new (unseen) classes as training progresses. Bottom: The robust–clean loss gap remains nearly constant on base classes but grows steadily on new classes. 
}
\Description{A two-panel line graph illustrates loss metric changes over 200 training epochs. The vertical axis of the top panel ranges from 0.0 to 3.0. A solid line representing robust training loss on base classes steadily decreases. Dashed lines for clean test loss on base and new classes remain near zero at the bottom. The solid line for robust test loss on base classes is stable around 1.3. In contrast, the solid line for robust test loss on new classes rises continuously from 1.6 to 2.2. The vertical axis of the bottom panel plots the loss gap. The gap between robust and clean test loss on base classes is a horizontal line near 0.9. The gap for new classes climbs steadily from 0.9 to 1.3 and deviates significantly.}
\label{fig:introduction}
\end{figure}

To defend against such threats, Adversarial Training (AT) \citep{adver_training} is widely regarded as the most effective defense. Prior works \citep{tecoa,fare,pmgaft,slade,advsimplex} have integrated AT with fine-tuning to enhance the robustness of CLIP-style models. However, given the ever-increasing scale of modern VLMs, the prohibitive computational cost of full-parameter adversarial fine-tuning renders it impractical for large-scale deployment. Consequently, Parameter-Efficient Fine-Tuning (PEFT) \citep{peft}  has emerged as a promising alternative. Among PEFT techniques, prompt tuning \citep{prompt_tuning} stands out for its lightweight nature and compatibility with large-scale image-text pre-training. Building on this, recent studies \citep{apt,advpt,fap,tapt,rtpt} have integrated AT with prompt tuning, termed \emph{adversarial prompt tuning}, to enhance the robustness of CLIP-style models on downstream tasks efficiently.

Upon analyzing the training dynamics of adversarial prompt tuning, we identify a phenomenon that diverges from the widely known robust overfitting \citep{robust_overfitting}, which manifests itself as a continued decrease in robust training loss coupled with an increase in robust test loss on the same classes as training progresses. 
In contrast, within the context of adversarial prompt tuning, we observe a pattern illustrated in \cref{fig:introduction}. Specifically, the robust training and test losses on base (seen) classes consistently decrease or remain stable throughout training. However, the robust test loss on new (unseen) classes begins to rise significantly when the training process proceeds. 
Crucially, the gap between robust and clean test losses remains stable for base classes, whereas it widens continuously for new classes as training progresses. We term this phenomenon ``\textit{robust generalization overfitting}'', highlighting a specific failure in transferring robustness to unseen semantic categories.

We attribute this robust generalization overfitting to shortcut learning \citep{shortcut}. Specifically, the model tends to rely on pseudo-robust features that satisfy the training objective for seen categories but lack high-level semantic generalization. Although adversarial training aims to encourage the model to acquire robust representations, the model often takes a shortcut by learning to reverse-engineer the specific perturbation patterns generated by the attack algorithm on base classes. Consequently, these features are effective for defending against attacks on seen data but are inherently non-generalizable to unseen categories, which limits the practical utility of adversarially trained VLMs.

Motivated by these insights, we propose \textbf{A}dversarial \textbf{D}isent\-\textbf{A}ngled \textbf{P}rompt \textbf{T}uning (\textbf{ADAPT}), a robust prompt tuning framework grounded in the philosophy of ``Learning What Not to Learn'' \citep{kim2019learning}. ADAPT designs a dual-prompt mechanism comprising a robust target prompt and a pool of decoy prompts. During training, decoy prompts are sampled from the pool and intentionally guided to absorb diverse non-generalizable pseudo-robust features, acting as a trap for shortcuts. Meanwhile, the target prompt is tasked with learning robust representations. To ensure the purity of the learned features, we enforce an explicit orthogonal loss that separates the target prompt from the decoy prompts. By effectively rejecting the shortcut features captured by the decoy prompts, the target prompt is compelled to focus on essential and generalizable visual semantics that remain robust across unseen categories. Extensive experiments demonstrate that the proposed ADAPT method effectively enhances the robustness of prompt on unseen classes.

The main contributions of this work are four-fold.
\begin{itemize}[topsep=2pt,itemsep=1pt,parsep=0pt,partopsep=0pt]
\item To the best of our knowledge, we are the first to find the phenomenon of ``robust generalization overfitting'' in adversarial prompt tuning, revealing that the degradation of robustness on unseen classes stems from the reliance of VLMs on pseudo-robust features.
\item We propose the ADAPT method as a robust prompt tuning framework, which employs a dual-prompt mechanism with orthogonal loss to actively entrap shortcuts into decoy prompts, thereby forcing the target prompt to learn robust features.
\item We provide a theoretical guarantee demonstrating that by minimizing the projection of the target prompt onto the shortcut subspace, ADAPT bounds the generalization error on unseen classes.
\item Extensive experiments on 15 datasets demonstrate the effectiveness of the proposed ADAPT method, particularly in the challenging adversarial base-to-new generalization setting, where it outperforms state-of-the-art methods by effectively mitigating the robustness gap between seen and unseen classes.
\end{itemize}

\section{Related Work}
\noindent \textbf{CLIP-based VLMs.}
VLMs have fundamentally reshaped cognitive systems by bridging the semantic gap between the visual and textual modalities \citep{yin1,yin2}, demonstrating exceptional capabilities in open-world vision tasks \citep{llava,minigpt}. The seminal work of CLIP \citep{clip}, pre-trained on about 400 million image-text pairs via contrastive learning, established a dominant paradigm for learning transferable vision-language representations. This success has catalyzed a proliferation of CLIP-style architectures. Following initial extensions like ALIGN \citep{align} and OpenCLIP \citep{openclip}, recent advancements have further expanded this paradigm through efficient scaling strategies such as SigLIP \citep{siglip}, enhanced data curation like MetaCLIP \citep{metaclip} and DFN \citep{dfn}, and massive parameter scaling exemplified by EVA-CLIP-18B \citep{evaclip,evaclip18}. Given its trailblazing role and widespread adoption as a standard backbone, our work focuses on CLIP to establish a baseline for robust prompt tuning within this foundational VLM paradigm.

\noindent \textbf{Adversarial training of VLMs.}
Despite the impressive zero-shot capabilities of VLMs, they remain highly vulnerable to adversarial attacks, which add imperceptible perturbations to input images, easily misleading model predictions \citep{explaining, schlarmann2023, zhao2023}. Adversarial Training (AT) \citep{adver_training} is widely recognized as one of the most effective defense strategies against such threats. Initial attempts to fortify VLMs, such as TeCoA \citep{tecoa}, FARE \citep{fare}, PMG-AFT \citep{pmgaft}, SLADE \citep{slade}, and AdvSimplex \citep{advsimplex}, relied on full-parameter adversarial fine-tuning. While effective, these methods incur prohibitively high computational costs, limiting large-scale deployment.

To mitigate these efficiency constraints, recent research has shifted towards integrating AT with prompt tuning \citep{coop}, a parameter-efficient fine-tuning paradigm originally designed to optimize continuous context vectors for downstream tasks \citep{coop, maple}. Building on this lightweight mechanism, adversarial prompt tuning methods have emerged as a promising direction. AdvPT \citep{advpt} and APT \citep{apt} demonstrate that optimizing prompts on adversarial examples can enhance the robustness. FAP \citep{fap} further improves the robustness by enforcing multi-modal consistency. Despite these advances, existing adversarial prompt tuning methods often struggle with robustness on unseen classes. The proposed ADAPT method aims to alleviate this problem through a dual-prompt mechanism that guides decoy prompts to absorb pseudo-robust features while enforcing orthogonality on the target prompt, thereby disentangling robust features from pseudo-robust features.

\section{Methodology}

In this section, we begin by overviewing the preliminaries and presenting a toy example to elucidate the phenomenon of robust generalization overfitting. Subsequently, we detail the proposed ADAPT framework and conclude with a theoretical analysis that substantiates its robustness.

\subsection{Preliminaries}
\noindent \textbf{CLIP.}
CLIP aligns visual and textual representations in a shared embedding space via contrastive pre-training. It consists of an image encoder $\mathcal{E}_\mathrm{I}$ and a text encoder $\mathcal{E}_\mathrm{T}$. Given an image $\mathbf{x}$, the image encoder extracts its visual feature $z( \mathbf{x} )  = \mathcal{E}_\mathrm{I}(\mathbf{x})$. For classification, a set of hand-crafted prompts is constructed using the manual template, e.g., ``a photo of a [CLS]'', where ``[CLS]'' is a placeholder for class tokens like ``dog''.\footnote{***you need to say the hand-crafted prompts are used as input to the text encoder. Otherwise why do you mention prompts here? revise this part. \edit{***revised.}} \edit{These prompts are then fed into the} text encoder to generate textual embeddings $
\mathbf{t}_{\mathrm{H}}=\{\mathbf{t}_{\mathrm{H}}^{c}\}_{c=1}^{C}
$ for each class $c$, where \edit{$C$ is the number of classes.}\footnote{***define $C$ here. \edit{***revised}} The probability that $\mathbf{x}$ belongs to class $y$ is calculated as
$
p_{\mathrm{clip}}(y|\mathbf{x},\mathbf{t}_{\mathrm{H}})=\frac{\exp(\cos(z( \mathbf{x} ) ,\mathbf{t}_{\mathrm{H}}^{y})/\tau )}{\sum\nolimits_{c=1}^C{\exp(\cos(z( \mathbf{x} ) ,\mathbf{t}_{\mathrm{H}}^{c})/\tau )}}
$,
where $\tau$ is the temperature and $\cos(\cdot, \cdot)$ denotes the cosine similarity.

\noindent \textbf{Prompt Tuning.}
While effective, hand-crafted prompts are suboptimal. Prompt tuning \citep{coop} addresses this by introducing continuous learnable vectors $\mathbf{v} = \{\mathbf{v}_1, \dots, \mathbf{v}_M\}$ to replace hand-crafted prompts. These vectors are concatenated with the class token to form the learnable prompt $\{\mathbf{v}, [ \mathrm{CLS}] \}$, which is then fed into $\mathcal{E}_T$ to obtain a set of textual embeddings $\mathbf{t}=\{\mathbf{t}^{c}\}_{c=1}^{C}$ with $\mathbf{t}^{c}$ denoting the textual embedding of the learnable prompt for class $c$. The probability of $\mathbf{x}$ belonging to class $y$ is calculated as
$p(y|\mathbf{x},\mathbf{t}) = \frac{\exp(\cos(z( \mathbf{x} ) , \mathbf{t}^y) / \tau)}{\sum\nolimits_{c=1}^C {\exp(\cos(z( \mathbf{x} ) , \mathbf{t}^c) / \tau)}}$.

\noindent \textbf{Adversarial Prompt Tuning (APT).}
Despite efficiency, standard prompt tuning methods remain vulnerable to adversarial attacks \citep{adver_training}. An adversarial example $\mathbf{x}_{\mathrm{adv}}$ in the adversarial attack is crafted by adding 
a perturbation $\delta$ to a clean image $\mathbf{x}$ within a perturbation budget $\epsilon$ by solving the following problem\footnote{***what is $\mathbf{t}$ here \edit{***$\mathbf{t}$ is the textual embedding of learnable prompts, consistent with the definition provided in the preceding ``Prompt Tuning'' paragraph.}}
\begin{equation*}
    \delta = \underset{\\ \|\delta'\|_p \le \epsilon}{\arg\max} \ \mathcal{L}_{\mathrm{ce}}(\mathbf{x}+\mathbf{\delta}', \mathbf{t}, y),
\end{equation*}
where $\mathbf{x}_{\mathrm{adv}} = \mathbf{x} + \delta$, $\|\cdot\|_p$ denotes the $\ell_p$ norm
(typically $\ell_\infty$), and 
$\mathcal{L}_{\mathrm{ce}} ( \mathbf{x}_{\mathrm{adv}},\mathbf{t},\mathbf{y} ) =-\sum\nolimits_{i=1}^C{y_i\log  p( y_i|\mathbf{x}_{\mathrm{adv}},\mathbf{t} )}$\footnote{***$\mathbf{t}$ is not used explicitly here. This cross-entropy loss is used multiple times in the subsequent sections and very important. Need to define it precisely. \edit{***I explicitly denote the probability as $p(y_i|\mathbf{x}, \mathbf{t})$ to clarify its dependency on the prompt embeddings $\mathbf{t}$.}} denotes the cross-entropy loss with ground-truth $\mathbf{y}$. 
To defend against such attacks, APT \citep{apt} integrates adversarial training into prompt tuning. That is, instead of training on clean images, APT optimizes the learnable vectors $\mathbf{v}$ to minimize the loss on adversarial examples as
\begin{equation}
\min_{\mathbf{v}} \ \mathbb{E}_{(\mathbf{x}, \mathbf{y}) \sim \mathcal{D}} \left[ \mathcal{L}_{\mathrm{ce}}(\mathbf{x}_{\mathrm{adv}}, \mathbf{t}, \mathbf{y}) \right],
\end{equation}
where $\mathcal{D}$ denotes the training dataset. Though APT improves the robustness on seen classes, our work identifies that it suffers from robust generalization overfitting.

\subsection{A Motivating Example}
To visualize the phenomenon of robust generalization overfitting, we construct a 2D synthetic dataset as shown in Figure \ref{fig:toy_example}.

\begin{figure}[tbph]
    \centering
    \includegraphics[width=1.0 \linewidth]{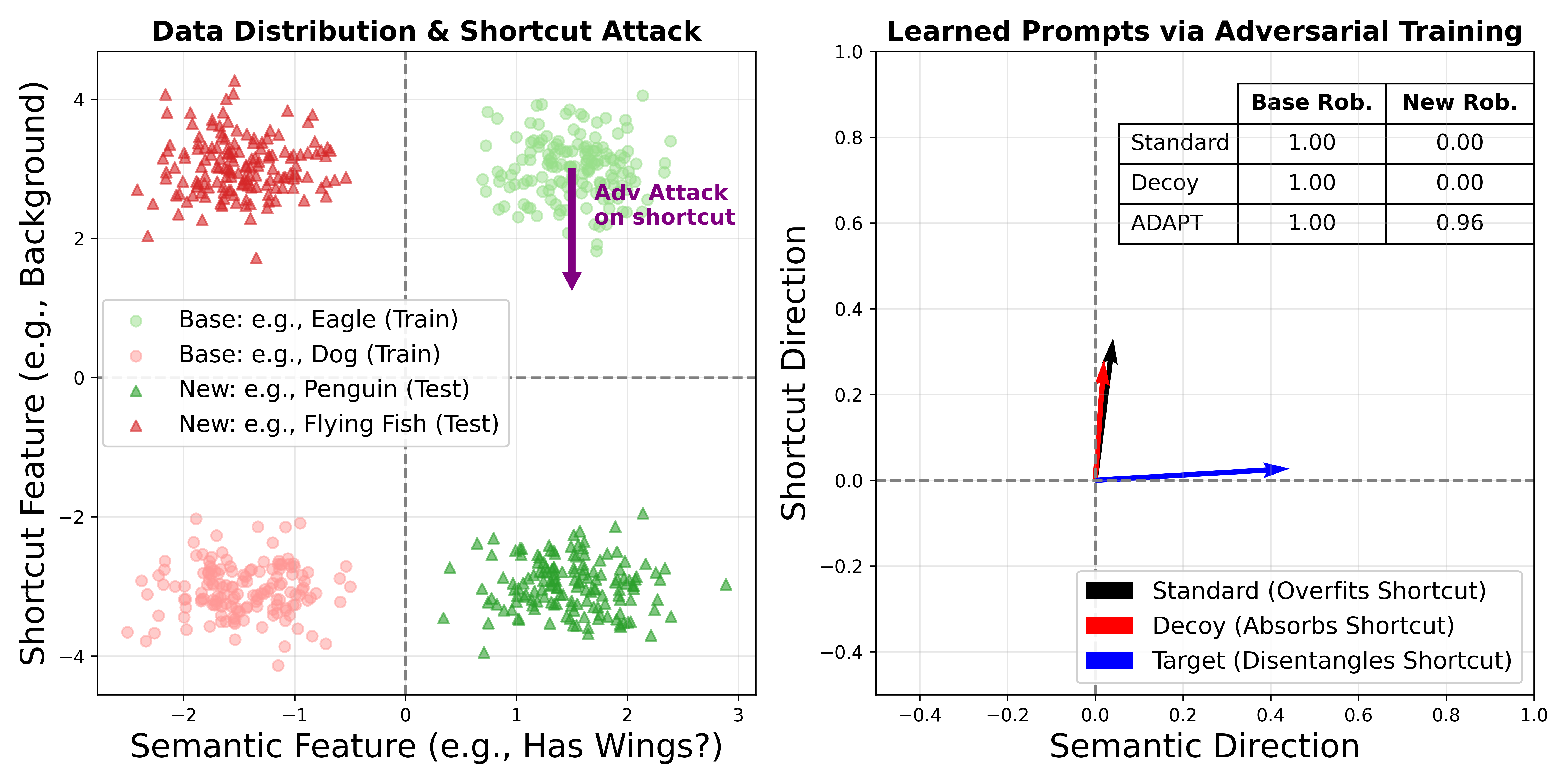}
    \caption{Toy example \edit{(bird vs. non-bird)}. While standard adversarial prompt tuning (i.e., black arrow) gets trapped in non-generalizable shortcuts, our ADAPT employs a decoy prompt (i.e., red arrow) to absorb these shortcuts, guiding the target prompt (i.e., blue arrow) to learn robust features.}
    \label{fig:toy_example}
    \Description{A two-panel diagram demonstrates how the proposed method avoids shortcut learning in a 2D feature space. The left panel is a scatter plot with semantic features on the horizontal axis and shortcut features on the vertical axis. It shows four data clusters with two for base classes and two for new classes. These clusters are separated by a large vertical gap representing the shortcut direction. The right panel displays three vector arrows originating from the center coordinates. A black arrow representing standard adversarial training points almost entirely along the vertical shortcut axis. A red arrow for the decoy prompt also points vertically to absorb this shortcut. A blue arrow for the target prompt points along the horizontal semantic direction and is orthogonal to the shortcut.}
\end{figure}

\noindent \textbf{Data Distribution and Shortcut Bias.}
We simulate a binary classification task with two latent features: a \textit{semantic feature} $x_{\mathrm{sem}}$ (x-axis) and a \textit{shortcut feature} $x_{\mathrm{sho}}$ (y-axis).
For the base classes, the data clusters are generated from Gaussian distributions with centroids at $[1.5, 3.0]$ (positive class) and $[-1.5, -3.0]$ (negative class), sharing a standard deviation of $0.4$.
Crucially, this setup creates a significantly larger margin \edit{(i.e., inter-class distance between centroids)}\footnote{***how to define the margin? \edit{***revised. We define the margin here as the inter-class distance between the centroids of the two classes along the respective axis.}} in the shortcut dimension (i.e., distance\footnote{***what distance? \edit{***see the last footnote.}} of $6.0$) compared to the semantic dimension (i.e., distance of $3.0$).
This disparity creates a strong inductive bias, as the classes are further apart along the shortcut axis, causing the model to prioritize $x_{\mathrm{sho}}$ as the primary discriminative feature to minimize the training loss.
Consequently, adversarial attacks (e.g., PGD attack \citep{adver_training}\footnote{***you did not mention this before. add references \edit{***revised.}}) naturally focus on this dominant direction, as perturbations along the axis with larger weights induce \edit{a larger increase in loss}\footnote{***loss increasing? \edit{***I replace ``induce the larger loss variation'' with ``induce a larger increase in loss''}}.

\noindent \textbf{Distribution Shift and Robustness Failure.}
For the new classes (test stage), we introduce a distribution shift where 
\edit{the centroid of the positive class shifts to $[1.5, -3.0]$ and that of the negative class changes to $[-1.5, 3.0]$.}\footnote{***this is not defined before. \edit{***I delete ``the spurious correlation is reversed:''}}
Crucially, the semantic feature $x_{\mathrm{sem}}$ remains invariant across this shift.
Figure \ref{fig:toy_example} (Right) visualizes the learned weight vectors (i.e., arrows) of different prompts. Note that the decision boundary is orthogonal to these weight vectors. \textit{Standard adversarial prompt tuning} (i.e., black arrow) is skewed heavily towards the y-axis, indicating it has overfitted the shortcut due to its larger margin.
While this yields high robustness on base classes, it leads to catastrophic failure (0\% robustness) on new classes, where the shortcut correlation is flipped.\footnote{***you could mention the robust generalization overfitting in this section. \edit{*** I added the blue sentence.}} \edit{This failure serves as a concrete manifestation of the robust generalization overfitting phenomenon.}

\noindent \textbf{Mechanism of ADAPT.}
Our ADAPT framework successfully disentangles these features.
The \textit{decoy prompt} (i.e., red arrow) is explicitly guided to align with the shortcut direction (i.e., y-axis), absorbing the non-robust features.
\edit{This alignment is enforced by a sparsity constraint (e.g., $\ell _1$ regularization) which forces the decoy prompt to learn the easiest shortcut feature. Since the shortcut dimension provides the most dominant gradient signal due to its larger margin, the decoy prompt greedily latches onto this single direction and suppresses the weaker semantic features.}
In contrast, constrained by the orthogonality to the decoy prompt, the \textit{target prompt} (i.e., blue arrow) is forced to align with the semantic direction (i.e., x-axis).
By effectively ignoring the volatile shortcut dimension, the target prompt achieves stable robustness on both base and new classes.
\footnote{***in this simple problem, it is possible that the target prompt is close to the y-axis and the decoy prompt is close to the x-axis since there is no hard prompt to regularize the learning of decoy prompt. \edit{***revised.}}

\begin{figure*}[t]
    \centering
    \includegraphics[width=0.9 \linewidth]{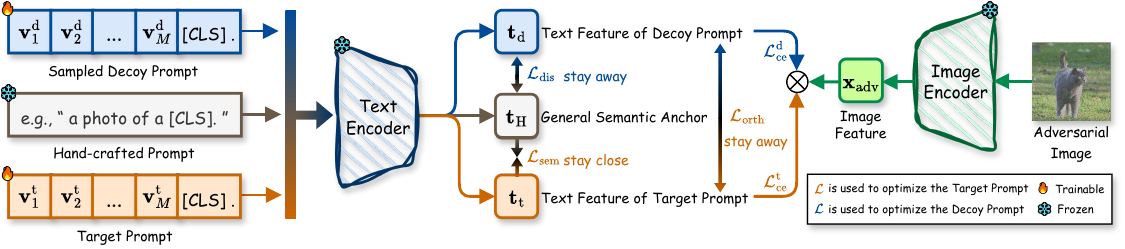}
    \caption{Overview of the proposed ADAPT method. ADAPT operates on a dual-prompt mechanism consisting of a \textit{Target Prompt} and a \textit{Decoy Prompt} (uniformly sampled from a pool).
The decoy prompt serves as a shortcut trap, optimized to absorb pseudo-robust features via $\mathcal{L}_\mathrm{ce}^\mathrm{d}$ while deviating from the general semantic anchor via $\mathcal{L}_\mathrm{dis}$.
Conversely, the target prompt is guided to learn robust features by enforcing orthogonality to the decoy prompt via $\mathcal{L}_\mathrm{orth}$ and maintaining semantic consistency with the general semantic anchor via $\mathcal{L}_\mathrm{sem}$.}
    \label{fig_framework}
\Description{A block diagram details the dual-prompt architecture and latent space loss constraints. On the left side, a sampled decoy prompt, a hand-crafted prompt, and a target prompt pass in parallel through a frozen text encoder. This generates three embeddings, including a decoy feature, a general semantic anchor, and a target feature. A dissimilarity loss pushes the decoy feature away from the semantic anchor. A semantic loss pulls the target feature toward the semantic anchor. An orthogonal loss forces the target feature away from the decoy feature. On the right side, an adversarial image passes through a frozen image encoder to generate an image feature. Independent cross-entropy losses then optimize this image feature against the decoy text feature and the target text feature.}
\end{figure*}

\subsection{ADAPT}
Motivated by the robust generalization overfitting phenomenon, we propose the ADAPT method. Unlike standard adversarial prompt tuning, which may bias the model \edit{to learn the pseudo-robust features},\footnote{***what do the shortcuts of seen classes mean? \edit{*** I replaced ``towards overfitting the shortcuts of seen classes'' with ``to learn the pseudo-robust features''. ``Shortcuts of seen classes'' refer to the non-generalizable features that are discriminative for the training (seen) categories but fail to transfer to unseen classes.}} ADAPT operates on a dual-prompt mechanism designed to disentangle robust features from pseudo-robust features. The overall architecture of ADAPT is illustrated in Figure \ref{fig_framework}.

\noindent \textbf{Dual-prompt mechanism.}
Existing APT learns a single \edit{learnable prompt}\footnote{***is it true? \edit{***I replaced ``continuous vector'' with ``learnable prompt''}}, which tends to learn pseudo-robust features that are discriminative for the adversarial examples of seen classes but are non-generalizable to unseen classes.
To address this, we introduce a \textbf{target prompt} and a \textbf{pool of decoy prompts}.
Specifically, the target prompt is defined as $\mathbf{P}_\mathrm{t} = \{ \mathbf{v}_1^\mathrm{t}, \dots, \mathbf{v}_M^\mathrm{t}, [\mathrm{CLS}] \}$.\footnote{***$\mathbf{v}_i^\mathrm{t}$ is a vector? There are $M$ vectors? If true, existing APT methods also learn $P_t$ which could include multiple continuous vector instead of a single one. \edit{***revised.}}
Considering that adversarial shortcuts may exhibit diverse patterns, a single decoy prompt may not suffice to capture all of them. Therefore, we construct a pool of $K$ decoy prompts $\mathcal{P}_\mathrm{d} = \{ \mathbf{P}_\mathrm{d}^1, \mathbf{P}_\mathrm{d}^2, \dots, \mathbf{P}_\mathrm{d}^K \}$.\footnote{***what is the form of each $\mathbf{P}_\mathrm{d}^i$? similar to $\mathbf{P}_t$? \edit{***revised.}} \edit{Each decoy prompt $\mathbf{P}_\mathrm{d}^k$ follows the same structure as the target prompt.} During inference, only the target prompt $\mathbf{P}_\mathrm{t}$ is used, ensuring no additional computational overhead compared to standard \edit{APT}\footnote{***adversarial prompt tuning? \edit{***revised.}}.

\noindent \textbf{Decoy prompt as a shortcut trap.}
We utilize decoy prompts as a trap to capture the pseudo-robust features.
In each training iteration, we uniformly sample an index $k \sim \mathcal{U}(1, K)$ to select an active decoy prompt $\mathbf{P}_\mathrm{d}^k$ from the pool. 
For the sake of brevity, we omit the index $k$ and refer to the currently active decoy prompt simply as $\mathbf{P}_\mathrm{d}=\{ \mathbf{v}_1^\mathrm{d}, \dots, \mathbf{v}_M^\mathrm{d}, [\mathrm{CLS}] \}$, and its corresponding textual embeddings to all classes as $\mathbf{t}_{\mathrm{d}}=\{\mathbf{t}_{\mathrm{d}}^{c}\}_{c=1}^{C}$.
It is worth noting that each decoy prompt in the pool is initialized independently and updated exclusively when selected.
This stochastic sampling strategy encourages the pool to learn a diverse set of shortcut directions, effectively covering the subspace of non-generalizable features.


To strictly enforce the decoy prompt to focus only on non-generalizable shortcuts, we aim to push the decoy prompt away from the general semantic space\footnote{***general semantic space seems undefined. \edit{***revised}} \edit{captured by hand-crafted prompts}. Specifically, we consider the text embeddings of the hand-crafted prompt as the general semantic anchor, denoted as $
\mathbf{t}_{\mathrm{H}}=\{\mathbf{t}_{\mathrm{H}}^{c}\}_{c=1}^{C}
$. 
Then, we explicitly maximize the dissimilarity between embeddings of the decoy prompt $\mathbf{t}_{\mathrm{d}}$ and the hand-crafted prompt $\mathbf{t}_{\mathrm{H}}$.
This can be achieved by minimizing the dissimilarity loss as
\begin{equation}
\mathcal{L}_\mathrm{dis}=\cos ( \mathbf{t}_{\mathrm{d}},\mathbf{t}_{\mathrm{H}}).
\label{eq_dis}
\end{equation}
Note that by minimizing Eq. \eqref{eq_dis}, we encourage the decoy prompt to deviate from the general semantic directions.
Simultaneously, the decoy prompt is required to minimize the classification loss on adversarial examples with the loss function as
\begin{equation}
\mathcal{L}_{\mathrm{ce}}^\mathrm{d} = \mathcal{L}_{\mathrm{ce}}(\mathbf{x}_{\mathrm{adv}}, \mathbf{t}_{\mathrm{d}}, \mathbf{y}).
\label{eq_decoy_ce}
\end{equation}
It is important to note that the adversarial examples $\mathbf{x}_{\mathrm{adv}}$ used here are generated at the beginning of the iteration using the current target prompt, thus capturing the specific perturbation patterns that successfully fool the target prompt.


By combining two losses in Eqs. \eqref{eq_dis} and \eqref{eq_decoy_ce}, the overall objective for learning the decoy prompt is formulated as
\begin{equation}
\mathcal{L}_\mathrm{total}^\mathrm{d} = \mathcal{L}_{\mathrm{ce}}^\mathrm{d} + \alpha \mathcal{L}_\mathrm{dis}.
\label{eq_decoy_total}
\end{equation}
By minimizing Eq.~\eqref{eq_decoy_total}, the decoy prompt is compelled to find a solution that satisfies the training objective (e.g., high robustness on seen classes) while being semantically distinct from general knowledge. Consequently, the decoy prompt becomes a dedicated trap for adversarial shortcuts.

\noindent \textbf{Learning target prompt via disentanglement.}
While the decoy prompt absorbs the shortcuts, the target prompt is guided to learn robust features. To achieve this, we impose an explicit orthogonal loss\footnote{***rigorously this cannot be termed constraint since this is only a regularizer to enforce to be close to orthogonality. \edit{***revised. I replace most constraint with loss.}} in the latent space.
Specifically, let $\mathbf{t}_{\mathrm{t}}=\{\mathbf{t}_{\mathrm{t}}^{c}\}_{c=1}^{C}$ denote the textual embeddings of the target prompt.
\edit{To capture the dominant optimization direction shared across classes,} we compute the global representation of the target and decoy prompts by averaging the text embeddings across all classes as $\mathbf{\bar{t}}_\mathrm{t} = \frac{1}{C} \sum_{c=1}^{C} \mathbf{t}_\mathrm{t}^c$ and $\mathbf{\bar{t}}_\mathrm{d} = \frac{1}{C} \sum_{c=1}^{C} \mathbf{t}_\mathrm{d}^c$. 
The orthogonal loss is formulated as\footnote{***why use global representation to define this loss \edit{***Pseudo-robust features (e.g., texture bias) often manifest as a systematic direction shared across classes. Averaging extracts this principal direction, allowing us to disentangle the prompt subspaces globally. \textbf{I added a blue sentence.}}}
\begin{equation}
\mathcal{L}_{\mathrm{orth}} = \cos^2(\mathbf{\bar{t}}_\mathrm{t}, \mathrm{sg}(\mathbf{\bar{t}}_\mathrm{d})),
\label{eq_orth}
\end{equation}
where $\mathrm{sg}(\cdot)$ denotes the stop-gradient operator. 
When Eq.~\eqref{eq_orth} reaches its minimum (i.e., 0), $\mathbf{\bar{t}}_\mathrm{t}$ becomes orthogonal to $\mathbf{\bar{t}}_\mathrm{d}$, achieving disentanglement between the target and decoy prompts. 
Here, we employ the stop-gradient operation on the decoy prompt, which ensures that the disentanglement process solely updates the target prompt without the interference of updating the decoy prompt.

Moreover, to maintain the semantic validity of the target prompt, we tether the target prompt to the general semantic anchor. 
Specifically, we can minimize the $\ell _1$ distance between the target prompt and the semantic anchor with the semantic loss formulated as
\begin{equation}
\mathcal{L}_{\mathrm{sem}} = \| \mathbf{t}_{\mathrm{t}} - \mathbf{t}_{\mathrm{H}} \|_1,
\label{eq_anchor}
\end{equation}
where $\|\cdot\|_1$ denotes the $\ell _1$ norm.
This loss ensures that the target prompt retains high-level semantic consistency with the hand-crafted prompt. 

By combining Eqs.~\eqref{eq_orth} and \eqref{eq_anchor} as well as the cross-entropy loss for classification, the overall objective for learning the target prompt is formulated as\footnote{***it is better to change $\lambda$ to $\gamma$ if $\gamma$ is not used elsewhere. \edit{$\gamma$ is used in Theorem}}
\begin{equation}
\mathcal{L}_{\mathrm{total}}^\mathrm{t} = \mathcal{L}_{\mathrm{ce}}^\mathrm{t} + \beta \mathcal{L}_{\mathrm{sem}} + \lambda \mathcal{L}_{\mathrm{orth}},
\label{eq_target_total}
\end{equation}
where $\mathcal{L}_{\mathrm{ce}}^\mathrm{t} = \mathcal{L}_{\mathrm{ce}}(\mathbf{x}_{\mathrm{adv}}, \mathbf{t}_{\mathrm{t}}, \mathbf{y})$. 

In summary, for the entire training process, the ADAPT method \edit{adopts an alternating optimization strategy that minimizes Eqs. \eqref{eq_decoy_total} and \eqref{eq_target_total} to update the sampled decoy prompt and the target prompt, respectively.}
The complete algorithm for the proposed ADAPT method is shown in Algorithm \ref{alg:ADAPT}.

\begin{algorithm}
    \caption{Adversarial DisentAngled Prompt Tuning (ADAPT)}
    \label{alg:ADAPT}
    \begin{algorithmic}[1]
        \STATE \textbf{Input:} Training dataset $\mathcal{D}$, CLIP image encoder $\mathcal{E}_\mathrm{I}$ and text encoder $\mathcal{E}_\mathrm{T}$, textual embeddings of hand-crafted prompts $\mathbf{t}_{\mathrm{H}}$, learning rate $\eta$, perturbation budget $\epsilon$, hyper-parameters $\alpha, \beta, \lambda$.
        \STATE \textbf{Initialize:} Target prompt $\mathbf{P}_\mathrm{t}$, a pool of decoy prompts $\mathcal{P}_\mathrm{d} = \{ \mathbf{P}_\mathrm{d}^1, \dots, \mathbf{P}_\mathrm{d}^K \}$.
        \FOR {each epoch}
            \FOR {minibatch $\mathbf{x}, \mathbf{y}$ in $\mathbb{D}$}
                \STATE \textcolor{gray}{// 1. Generate Adversarial Examples using Target Prompt}
                \STATE Compute textual embeddings of target prompt $\mathbf{t}_{\mathrm{t}} = \mathcal{E}_\mathrm{T}(\mathbf{P}_\mathrm{t})$;
                
                \STATE Generate adversarial examples:
                \STATE \quad $\mathbf{x}_{\mathrm{adv}}\gets\underset{\mathbf{x}_{\mathrm{adv}}}{\mathrm{arg}\max}\,\mathcal{L} _{\mathrm{ce}}( \mathbf{x}_{\mathrm{adv}},\mathbf{t}_{\mathrm{t}},\mathbf{y} ) ,\,\,\mathrm{s}.\mathrm{t}.\parallel \mathbf{x}_{\mathrm{adv}}-\mathbf{x}\parallel _p\leqslant \epsilon$;
                
                \STATE \textcolor{gray}{// 2. Update Sampled Decoy Prompt}
                \STATE Uniformly sample an index $k \sim \mathcal{U}(1, K)$;
                \STATE Get the active decoy prompt $\mathbf{P}_\mathrm{d}^k$ and compute $\mathbf{t}_{\mathrm{d}} = \mathcal{E}_\mathrm{T}(\mathbf{P}_\mathrm{d}^k)$;
                
                \STATE Calculate loss of the decoy prompt: 
                \STATE \quad $\mathcal{L}_\mathrm{total}^\mathrm{d} = \mathcal{L}_{\mathrm{ce}}(\mathbf{x}_{\mathrm{adv}}, \mathbf{t}_{\mathrm{d}}, \mathbf{y}) + \alpha \cos(\mathbf{t}_{\mathrm{d}}, \mathbf{t}_{\mathrm{H}})$
                
                \STATE Update $\mathbf{P}_\mathrm{d}^k \leftarrow \mathbf{P}_\mathrm{d}^k - \eta \nabla_{\mathbf{P}_\mathrm{d}^k} \mathcal{L}_\mathrm{total}^\mathrm{d}$;
                
                \STATE \textcolor{gray}{// 3. Update Target Prompt}
                \STATE Calculate loss of the target prompt:
                \STATE \quad $\mathcal{L}_{\mathrm{total}}^\mathrm{t} = \mathcal{L}_{\mathrm{ce}}(\mathbf{x}_{\mathrm{adv}}, \mathbf{t}_{\mathrm{t}}, \mathbf{y}) + \beta \|\mathbf{t}_{\mathrm{t}} - \mathbf{t}_{\mathrm{H}}\|_1 + \lambda \cos^2(\mathbf{\bar{t}}_\mathrm{t}, \mathrm{sg}(\mathbf{\bar{t}}_\mathrm{d}))$
                
                \STATE Update $\mathbf{P}_\mathrm{t} \leftarrow \mathbf{P}_\mathrm{t} - \eta \nabla_{\mathbf{P}_\mathrm{t}} \mathcal{L}_{\mathrm{total}}^\mathrm{t}$;
            \ENDFOR
        \ENDFOR
    \end{algorithmic}
\end{algorithm}

\subsection{Theoretical Analysis} 

In this section, we analyze whether loss functions in ADAPT, the semantic loss in Eq.~\eqref{eq_anchor} and the orthogonal loss in Eq.~\eqref{eq_orth}, can guarantee that the target classifier remains robust to shortcut shifts\footnote{***what do shortcut shifts mean? \edit{***shortcut shifts refers to a change in the distribution of non-semantic features used for prediction from base to new classes.}} on unseen classes.

\edit{All CLIP image/text embeddings are $\ell_2$-normalized in this section unless stated otherwise.}
Consider that the feature space $\mathbb{R}^d$ can be decomposed into two orthogonal subspaces: a semantic subspace $\mathcal{S}$ and a shortcut subspace $\mathcal{U}$ (i.e., $\mathbb{R}^d = \mathcal{S} \oplus \mathcal{U}$). Let the image feature\footnote{***what does normalized mean? why need the normalized features in the analysis but not in the method?} of an input $\mathbf{x}$ be decomposed as ${z}(\mathbf{x}) = s(\mathbf{x}) + u(\mathbf{x})$, where $s(\mathbf{x}) \in \mathcal{S}$ represents invariant semantics and $u(\mathbf{x}) \in \mathcal{U}$ represents distinct shortcut patterns (with $\|u(\mathbf{x})\|_2 \le B$). We assume that the semantic anchors $\{{\mathbf{t}}_{\mathrm{H}}^{c}\}_{c=1}^{C}$\footnote{***why use a different notation $\tilde{\mathbf{t}}_{\mathrm{H}}^{c}$ \edit{I add a declaration ``All CLIP image/text embeddings are $\ell_2$-normalized unless stated otherwise.'', and use the same notation as in the method.}} lie in $\mathcal{S}$ and provide a sufficient margin for classification as in the following assumption.
\begin{assumption}[Semantic Separability]
\label{assum:margin}
There exists a margin $\gamma > 0$ such that for any new-class sample $(\mathbf{x}, y)$, the semantic component $s(\mathbf{x})$ is correctly classified by the anchors with probability at least $1-\zeta$. Specifically, the following inequality holds with probability at least $1-\zeta$\footnote{***I added. is it what you want to express?}
\begin{equation}
    m_{\mathrm{H}}(\mathbf{x},y) :=\,\,\min_{c\ne y} \,\,\langle s(\mathbf{x}),\,{\mathbf{t}}_{\mathrm{H}}^{y}-{\mathbf{t}}_{\mathrm{H}}^{c}\rangle \,\,\ge \,\,\gamma,
\end{equation}
\end{assumption}
\edit{where $< \cdot ,\cdot >$ denotes the inner product operation.}

Under such settings, Theorem~\ref{thm:adapt_main} (proved in Appendix~\ref{app:theory}) shows that if the target prompt is sufficiently close to the semantic anchor (enforced by $\mathcal{L}_\mathrm{sem}$) and orthogonal to the shortcut subspace (enforced by $\mathcal{L}_\mathrm{orth}$), the robustness \edit{carries over to the target classifier on unseen classes}.\footnote{***what does `transfer' mean? \edit{The target classifier makes the same correct prediction when the shortcut component $u(\mathbf{x})$ shifts (within $\|u(\mathbf{x})\|_2\le B$), yielding a new-class error bound of at most $\delta$.}}

\begin{theorem}[Robustness Generalization]
\label{thm:adapt_main}
Define the semantic deviation $\varepsilon_{\mathrm{sem}} := \max_{c} \|{\mathbf{t}}_{\mathrm{t}}^{c} - {\mathbf{t}}_{\mathrm{H}}^{c}\|_2$ and the shortcut leakage $\varepsilon_{\mathcal{U}} := \max_{c} \|P_{\mathcal{U}}({\mathbf{t}}_{\mathrm{t}}^{c})\|_2$,\footnote{***I changed $P_{\mathcal{U}}$ to a function. You need to change it in elsewhere it appears, for example, the proof. \edit{***revised.}}\footnote{***Are $\varepsilon_{\mathrm{sem}}$ and $\varepsilon_{\mathcal{U}}$ related to the two losses? seems to use different norms. \edit{***$\mathcal{L}_{\mathrm{sem}}$ (Eq.~\ref{eq_anchor}) encourages small $\varepsilon_{\mathrm{sem}}$, using l2-norm here is to facilitate its compatibility with Cauchy–Schwarz,
while $\mathcal{L}_{\mathrm{orth}}$ (Eq.~\ref{eq_orth}) reduces the projection of target prompts onto shortcut directions,
captured by $\varepsilon_{\mathcal{U}}$.}} where $P_{\mathcal{U}}(\cdot)$ is to project onto $\mathcal{U}$. 
If the \edit{target} prompts\footnote{***which prompts? \edit{***revised}} satisfy
\begin{equation}
    \gamma > 2\varepsilon_{\mathrm{sem}} + 2B\varepsilon_{\mathcal{U}},
\end{equation}
then the target classifier $\hat{y}(\mathbf{x}) = \arg\max_{c} \langle {z}(\mathbf{x}), {\mathbf{t}}_{\mathrm{t}}^{c} \rangle$ is invariant\footnote{***what does invariant mean? \edit{***I add an explanation}} to any shortcut shifts\footnote{***shortcut shift or shortcut? \edit{***I add an explanation}}, \edit{in the sense that its prediction does not change under arbitrary variations of the shortcut component}. Consequently, the \edit{testing} error rate\footnote{***what error rate? training or testing? \edit{revised.}} on new classes is bounded by $\zeta$.
\end{theorem}

\begin{table*}[!t]
\renewcommand{\arraystretch}{1}
\centering
\caption{Performance under the setting of adversarial base-to-new generalization. All methods (except zero-shot TeCoA) are trained with 16 instances per base class. Here $\epsilon$ equals $4/255$. Results of $\epsilon=1/255$ are provided in Appendix \ref{appendix:b2n}.}
\captionsetup[sub]{skip=1pt}
\subfloat[\textbf{Average over 11 datasets}]{
    \begin{minipage}[t]{0.32\textwidth}
    \scalebox{0.72}{
    \setlength{\tabcolsep}{1.5mm}{
    \begin{tabular}{l|cc|cc|c}
    \toprule
    \multicolumn{1}{c|}{\multirow{2}[4]{*}{Methods}} & \multicolumn{2}{c|}{Base} & \multicolumn{2}{c|}{New} & \multirow{2}[4]{*}{$\mathrm{H}_\mathrm{b}$} \\
\cmidrule{2-5}          & Acc.   & Rob.   & Acc.   & Rob.   &  \\
    \midrule
    TeCoA & 44.51  & 14.62  & 42.16 & 16.14  & 22.66  \\
    \midrule
    APT   & 58.62  & 27.40  & 30.13  & 10.96  & 22.47  \\
    \rowcolor{Gray} ADAPT & 60.04 & 27.48  & 38.98  & 15.65  & 28.05  \\
    \midrule
    FAP   & 58.00  & 27.96  & 39.77  & 16.17  & 28.57  \\
    \rowcolor{Gray} ADAPT$_{\text M}$ & 56.82  & 28.52 & 41.39  & 17.76 & \textbf{30.05} \\
    \bottomrule
    \end{tabular}}}
    \end{minipage}
}
\hfill
\subfloat[ImageNet]{
    \begin{minipage}[t]{0.32\textwidth}
    \scalebox{0.72}{
    \setlength{\tabcolsep}{1.5mm}{
    \begin{tabular}{l|cc|cc|c}
    \toprule
    \multicolumn{1}{c|}{\multirow{2}[4]{*}{Methods}} & \multicolumn{2}{c|}{Base} & \multicolumn{2}{c|}{New} & \multirow{2}[4]{*}{$\mathrm{H}_\mathrm{b}$} \\
\cmidrule{2-5}          & Acc.   & Rob.   & Acc.   & Rob.   &  \\
    \midrule
    TeCoA & 43.65  & 11.36  & {44.95} & 13.67  & 19.39  \\
    \midrule
    APT   & 44.54  & 13.78  & 38.56  & 12.47  & 19.89  \\
    \rowcolor{Gray} ADAPT & {47.07} & 13.49  & 43.36  & 13.71  & 20.90  \\
    \midrule
    FAP   & 43.29  & {13.89} & 37.43  & 13.05  & 20.16  \\
    \rowcolor{Gray} ADAPT$_{\text M}$ & 41.67  & 13.88  & 39.47  & {14.49} & \textbf{21.01} \\
    \bottomrule
    \end{tabular}}}
    \end{minipage}
}
\hfill
\subfloat[Caltech101]{
    \begin{minipage}[t]{0.32\textwidth}
    \scalebox{0.72}{
    \setlength{\tabcolsep}{1.5mm}{
    \begin{tabular}{l|cc|cc|c}
    \toprule
    \multicolumn{1}{c|}{\multirow{2}[4]{*}{Methods}} & \multicolumn{2}{c|}{Base} & \multicolumn{2}{c|}{New} & \multirow{2}[4]{*}{$\mathrm{H}_\mathrm{b}$} \\
\cmidrule{2-5}          & Acc.   & Rob.   & Acc.   & Rob.   &  \\
    \midrule
    TeCoA & 85.41  & 49.45  & {83.08} & {54.37} & 64.14  \\
    \midrule
    APT   & 90.77  & 65.07  & 69.21  & 42.03  & 61.89  \\
    \rowcolor{Gray} ADAPT & {91.28} & 65.91  & 82.21  & 51.42  & \textbf{69.28} \\
    \midrule
    FAP   & 85.22  & 63.72  & 70.20  & 44.87  & 62.54  \\
    \rowcolor{Gray} ADAPT$_{\text M}$ & 89.74  & {66.24} & 79.15  & 52.29  & 68.97  \\
    \bottomrule
    \end{tabular}}}
    \end{minipage}
}
\\
\subfloat[OxfordPets]{
    \begin{minipage}[t]{0.32\textwidth}
    \scalebox{0.72}{
    \setlength{\tabcolsep}{1.5mm}{
    \begin{tabular}{l|cc|cc|c}
    \toprule
    \multicolumn{1}{c|}{\multirow{2}[4]{*}{Methods}} & \multicolumn{2}{c|}{Base} & \multicolumn{2}{c|}{New} & \multirow{2}[4]{*}{$\mathrm{H}_\mathrm{b}$} \\
\cmidrule{2-5}          & Acc.   & Rob.   & Acc.   & Rob.   &  \\
    \midrule
    TeCoA & 74.11  & 20.15  & {82.10} & 29.25  & 36.53  \\
    \midrule
    APT   & 74.22  & 27.38  & 73.04  & 23.21  & 37.46  \\
    \rowcolor{Gray} ADAPT & {79.00} & 27.59  & 76.17  & 27.96  & 40.90  \\
    \midrule
    FAP   & 74.43  & {31.84} & 68.51  & 26.51  & 41.17  \\
    \rowcolor{Gray} ADAPT$_{\text M}$ & 77.35  & 31.53  & 74.22  & {33.00} & \textbf{45.24} \\
    \bottomrule
    \end{tabular}}}
    \end{minipage}
}
\hfill
\subfloat[StanfordCars]{
    \begin{minipage}[t]{0.32\textwidth}
    \scalebox{0.72}{
    \setlength{\tabcolsep}{1.5mm}{
    \begin{tabular}{l|cc|cc|c}
    \toprule
    \multicolumn{1}{c|}{\multirow{2}[4]{*}{Methods}} & \multicolumn{2}{c|}{Base} & \multicolumn{2}{c|}{New} & \multirow{2}[4]{*}{$\mathrm{H}_\mathrm{b}$} \\
\cmidrule{2-5}          & Acc.   & Rob.   & Acc.   & Rob.   &  \\
    \midrule
    TeCoA & 12.49  & 2.10  & 17.38  & 2.08  & 3.65  \\
    \midrule
    APT   & 40.23  & {10.27} & 15.72  & 3.14  & 7.93  \\
    \rowcolor{Gray} ADAPT & 42.93  & 9.92  & 19.76  & 3.79  & 9.12  \\
    \midrule
    FAP   & {45.75} & 8.87  & 34.41  & 5.64  & 11.73  \\
    \rowcolor{Gray} ADAPT$_{\text M}$ & 38.81  & 8.10  & {36.35} & {7.18} & \textbf{12.66} \\
    \bottomrule
    \end{tabular}}}
    \end{minipage}
}
\hfill
\subfloat[Flowers102]{
    \begin{minipage}[t]{0.32\textwidth}
    \scalebox{0.72}{
    \setlength{\tabcolsep}{1.5mm}{
    \begin{tabular}{l|cc|cc|c}
    \toprule
    \multicolumn{1}{c|}{\multirow{2}[4]{*}{Methods}} & \multicolumn{2}{c|}{Base} & \multicolumn{2}{c|}{New} & \multirow{2}[4]{*}{$\mathrm{H}_\mathrm{b}$} \\
\cmidrule{2-5}          & Acc.   & Rob.   & Acc.   & Rob.   &  \\
    \midrule
    TeCoA & 40.65  & 14.81  & {38.94} & {10.28} & 18.60  \\
    \midrule
    APT   & {84.90} & {49.57} & 19.65  & 4.75  & 13.63  \\
    \rowcolor{Gray} ADAPT & 82.15  & 48.53  & 24.18  & 6.88  & 18.22  \\
    \midrule
    FAP   & 75.31  & 45.96  & 27.73  & 7.94  & 20.30  \\
    \rowcolor{Gray} ADAPT$_{\text M}$ & 78.73  & 45.39  & 28.79  & 9.50  & \textbf{22.89} \\
    \bottomrule
    \end{tabular}}}
    \end{minipage}
}
\hfill
\\
\subfloat[Food101]{
    \begin{minipage}[t]{0.32\textwidth}
    \scalebox{0.72}{
    \setlength{\tabcolsep}{1.5mm}{
    \begin{tabular}{l|cc|cc|c}
    \toprule
    \multicolumn{1}{c|}{\multirow{2}[4]{*}{Methods}} & \multicolumn{2}{c|}{Base} & \multicolumn{2}{c|}{New} & \multirow{2}[4]{*}{$\mathrm{H}_\mathrm{b}$} \\
\cmidrule{2-5}          & Acc.   & Rob.   & Acc.   & Rob.   &  \\
    \midrule
    TeCoA & {59.49} & 4.92  & 32.19  & 5.61  & 9.32  \\
    \midrule
    APT   & 37.22  & 10.11  & 18.61  & 3.79  & 9.02  \\
    \rowcolor{Gray} ADAPT & 42.37  & 10.69  & 30.09  & 6.78  & 13.43  \\
    \midrule
    FAP   & 50.39  & 11.92  & {42.72} & {10.37} & \textbf{17.89} \\
    \rowcolor{Gray} ADAPT$_{\text M}$ & 41.39  & {13.33} & 31.52  & 9.49  & 16.93  \\
    \bottomrule
    \end{tabular}}}
    \end{minipage}
}
\hfill
\subfloat[FGVCAircraft]{
    \begin{minipage}[t]{0.32\textwidth}
    \scalebox{0.72}{
    \setlength{\tabcolsep}{1.5mm}{
    \begin{tabular}{l|cc|cc|c}
    \toprule
    \multicolumn{1}{c|}{\multirow{2}[4]{*}{Methods}} & \multicolumn{2}{c|}{Base} & \multicolumn{2}{c|}{New} & \multirow{2}[4]{*}{$\mathrm{H}_\mathrm{b}$} \\
\cmidrule{2-5}          & Acc.   & Rob.   & Acc.   & Rob.   &  \\
    \midrule
    TeCoA & 10.26  & { 0.66}  & 9.96  & 0.90  & { 1.42}  \\
    \midrule
    APT   & {21.19} & { 7.08}  & 5.82  & 2.22  & { 4.93}  \\
    \rowcolor{Gray} ADAPT & 20.17  & { 6.36}  & 12.42  & 2.70  & { 6.08}  \\
    \midrule
    FAP   & 21.13  & { 6.90}  & {15.78} & {3.90} & { 7.81}  \\
    \rowcolor{Gray} ADAPT$_{\text M}$ & 19.93  & { 8.52} & 13.62  & 3.78  & \textbf{ 7.91} \\
    \bottomrule
    \end{tabular}}}
    \end{minipage}
}
\hfill
\subfloat[SUN397]{
    \begin{minipage}[t]{0.32\textwidth}
    \scalebox{0.72}{
    \setlength{\tabcolsep}{1.5mm}{
    \begin{tabular}{l|cc|cc|c}
    \toprule
    \multicolumn{1}{c|}{\multirow{2}[4]{*}{Methods}} & \multicolumn{2}{c|}{Base} & \multicolumn{2}{c|}{New} & \multirow{2}[4]{*}{$\mathrm{H}_\mathrm{b}$} \\
\cmidrule{2-5}          & Acc.   & Rob.   & Acc.   & Rob.   &  \\
    \midrule
    TeCoA & 39.81  & 7.91  & {45.36} & 10.84  & 15.05  \\
    \midrule
    APT   & 53.40  & 14.27  & 30.08  & 6.97  & 15.06  \\
    \rowcolor{Gray} ADAPT & {55.08} & 14.54  & 37.25  & 9.07  & 17.85  \\
    \midrule
    FAP   & 50.88  & {15.17} & 41.31  & 12.42  & 21.02  \\
    \rowcolor{Gray} ADAPT$_{\text M}$ & 48.35  & 14.40  & 44.05  & {13.55} & \textbf{21.43} \\
    \bottomrule
    \end{tabular}}}
    \end{minipage}
}
\hfill
\\
\subfloat[DTD]{
    \begin{minipage}[t]{0.32\textwidth}
    \scalebox{0.72}{
    \setlength{\tabcolsep}{1.5mm}{
    \begin{tabular}{l|cc|cc|c}
    \toprule
    \multicolumn{1}{c|}{\multirow{2}[4]{*}{Methods}} & \multicolumn{2}{c|}{Base} & \multicolumn{2}{c|}{New} & \multirow{2}[4]{*}{$\mathrm{H}_\mathrm{b}$} \\
\cmidrule{2-5}          & Acc.   & Rob.   & Acc.   & Rob.   &  \\
    \midrule
    TeCoA & 32.87  & 16.09  & 34.30  & 18.48  & 22.75  \\
    \midrule
    APT   & 53.36  & 25.35  & 23.67  & 10.75  & 20.68  \\
    \rowcolor{Gray} ADAPT & 54.63  & 28.01  & 30.31  & 15.22  & 26.19  \\
    \midrule
    FAP   & 54.17  & 28.94  & 26.93  & 15.10  & 25.58  \\
    \rowcolor{Gray} ADAPT$_{\text M}$ & {55.21} & {32.06} & {34.90} & {20.05} & \textbf{31.29} \\
    \bottomrule
    \end{tabular}}}
    \end{minipage}
}
\hfill
\subfloat[EuroSAT]{
    \begin{minipage}[t]{0.32\textwidth}
    \scalebox{0.72}{
    \setlength{\tabcolsep}{1.5mm}{
    \begin{tabular}{l|cc|cc|c}
    \toprule
    \multicolumn{1}{c|}{\multirow{2}[4]{*}{Methods}} & \multicolumn{2}{c|}{Base} & \multicolumn{2}{c|}{New} & \multirow{2}[4]{*}{$\mathrm{H}_\mathrm{b}$} \\
\cmidrule{2-5}          & Acc.   & Rob.   & Acc.   & Rob.   &  \\
    \midrule
    TeCoA & 48.19  & 23.48  & 30.85  & 22.36  & 28.47  \\
    \midrule
    APT   & {81.76} & 54.81  & 11.95  & 5.33  & 13.25  \\
    \rowcolor{Gray} ADAPT & 78.86  & 52.14  & 42.64  & 26.56  & 43.03  \\
    \midrule
    FAP   & 79.76  & 54.86  & {43.26} & {28.00} & \textbf{44.64} \\
    \rowcolor{Gray} ADAPT$_{\text M}$ & 76.29  & {55.43} & 37.18  & 20.05  & 37.06  \\
    \bottomrule
    \end{tabular}}}
    \end{minipage}
}
\hfill
\subfloat[UCF101]{
    \begin{minipage}[t]{0.32\textwidth}
    \scalebox{0.72}{
    \setlength{\tabcolsep}{1.5mm}{
    \begin{tabular}{l|cc|cc|c}
    \toprule
    \multicolumn{1}{c|}{\multirow{2}[4]{*}{Methods}} & \multicolumn{2}{c|}{Base} & \multicolumn{2}{c|}{New} & \multirow{2}[4]{*}{$\mathrm{H}_\mathrm{b}$} \\
\cmidrule{2-5}          & Acc.   & Rob.   & Acc.   & Rob.   &  \\
    \midrule
    TeCoA & 42.71  & 9.88  & {44.67} & 9.73  & 16.01  \\
    \midrule
    APT   & 63.19  & 23.73  & 25.15  & 5.90  & 14.97  \\
    \rowcolor{Gray} ADAPT & {66.86} & 25.13  & 30.34  & 8.11  & 18.96  \\
    \midrule
    FAP   & 57.70  & {25.49} & 29.15  & 10.06  & 21.02  \\
    \rowcolor{Gray} ADAPT$_{\text M}$ & 57.55  & 24.87  & 36.02  & {12.01} & \textbf{23.72} \\
    \bottomrule
    \end{tabular}}}
    \end{minipage}
}
\label{table_base2new}
\end{table*}

\section{Experiments}

\subsection{Setups}
\label{sec:setup}
\noindent \textbf{Datasets.}
Following the protocols established in \citep{coop,fap,apt}, we conduct experiments under four distinct settings: adversarial base-to-new generalization, adversarial few-shot classification, adversarial cross-dataset generalization, and adversarial domain generalization.
The first three settings are evaluated on a standard suite of 11 datasets. 
For the setting of adversarial domain generalization, we train on ImageNet \citep{imagenet} and evaluate on its out-of-distribution (OOD) variants. 
Details are provided in Appendix \ref{appendix:Implementation_Details}.

\noindent \textbf{Baselines.}
We compare ADAPT with state-of-the-art adversarial defense strategies for VLMs. Specifically, we select zero-shot TeCoA \citep{tecoa} as the representative for robust fine-tuning.
For adversarial prompt tuning\footnote{***adversarial prompt tuning? \edit{revised}} methods, we compare against APT \citep{apt}, which optimizes textual prompts, and FAP \citep{fap}, which incorporates multi-modal prompts.
These baselines cover the primary paradigms in robust VLM research.

Since ADAPT focuses on optimizing textual prompts, it has fewer learnable parameters than multi-modal prompt tuning methods such as FAP. For a fair comparison and to further explore the potential of our method, we extend our method to a multi-modal prompt version, termed ADAPT$_{\text M}$.\footnote{***what about ADAPT$_{\text M}$? \edit{***revised.}} Drawing inspiration from\footnote{***it is better not to mention this method as you did not compare with it\edit{***revised.}} FAP \citep{fap}, ADAPT$_{\text M}$ establishes a deep coupling mechanism between vision and language branches.
\edit{Specifically, ADAPT$_{\text M}$ maintains two sets of learnable deep visual prompts for the target and decoy prompts, respectively. For each layer, these visual prompts are mapped into the text embedding space through a lightweight cross-modal projection, producing the corresponding deep textual prompts. The deep visual and text prompts are then injected into the input sequence of the corresponding transformer block. In this way, the two branches are coupled layer-by-layer. Thus, both the target and the decoy have their own multi-modal prompts.}\footnote{***this part is unclear to me. need more details.\edit{***revised}}

\noindent \textbf{Adversarial training and evaluation.}
For adversarial training and evaluation, we employ PGD attack \citep{adver_training} under the $\ell_{\infty}$ threat model. Consistent with prior works \citep{apt,tecoa,fare}, we conduct experiments under two perturbation budgets of $\epsilon = 1/255$ and $\epsilon = 4/255$.
During training, we generate adversarial examples using a 3-step PGD with a step size of $2\epsilon/3$.
For evaluation, we employ a 100-step PGD with a step size of $\epsilon/4$ and random initialization.
More implementation details are provided in Appendix \ref{appendix:Implementation_Details}.\footnote{***you need to introduce the evaluation metrics especially `$\mathrm{H}_\mathrm{b}$' in somewhere of the main body. you cannot introduce them in the caption of tables or figures. \edit{***revised.}}

\subsection{Main Results}
\noindent \textbf{Adversarial base-to-new generalization.}
In this scenario, we evaluate the robust base-to-new transferability of the learned prompts by splitting datasets into disjoint base (seen) and new (unseen) classes. Models are trained solely on base classes and evaluated on both sets. Table \ref{table_base2new} reports the detailed performance across 11 datasets, \edit{where `$\mathrm{H}_\mathrm{b}$' denotes the harmonic mean computed over
accuracy and robustness on base classes and on new classes.}
We can see that APT suffers from severe robust generalization overfitting, characterized by high performance on base classes but a sharp decline on new classes.
In contrast, ADAPT effectively mitigates this issue.
On the challenging new classes, ADAPT outperforms APT by an average of 8.85\% in accuracy and 4.69\% in robustness, resulting in a substantial gain of 5.58\% in the `$\mathrm{H}_\mathrm{b}$' metric.
Furthermore, this advantage extends to the multi-modal prompting, where ADAPT$_{\text M}$ consistently surpasses the state-of-the-art FAP, achieving improvements of 1.62\% and 1.59\% in terms of the accuracy and robustness on new classes, respectively.
Crucially, the gains in robust generalization do not come at the cost of performance degradation on base classes. ADAPT achieves the highest average accuracy (i.e., 60.04\%) on base classes among all baseline methods, surpassing APT and FAP by 1.42\% and 2.04\%, respectively, and performs comparably in terms of robustness.
ADAPT$_{\text M}$ achieves the highest average robustness on base classes (i.e., 28.52\%).
Those results demonstrate that by effectively disentangling robust features from shortcuts, the proposed method achieves the best trade-off between accuracy and robustness on base and new classes.

\begin{figure}[!t]
\centering
\includegraphics[width=0.82\linewidth]{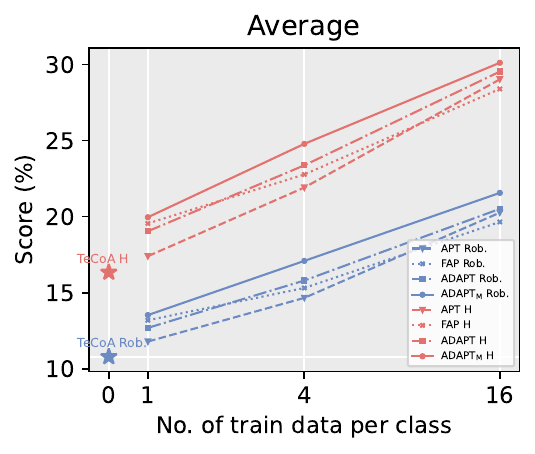}
\caption{Average performance on the 11 datasets under the adversarial few-shot classification setting. Since accuracies are comparable across methods, we plot only robustness and `H' for clarity. 
$\epsilon=4/255$. Full results are provided in Appendix \ref{appendix:fewshot}.}
\label{fig:few_shot}
\Description{A line chart compares robustness and harmonic mean scores across one, four, and sixteen-shot training scales. The horizontal axis uses a logarithmic scale for the amount of training data per class. The vertical axis measures percentage scores from 10 to 30. The data form two distinct upward-sloping bands. The upper red band represents the harmonic mean, while the lower blue band represents robustness. Within both regions, the proposed multi-modal ADAPT method occupies the top position and achieves the highest scores at all data scales. It is closely followed by the single-modal ADAPT method, FAP, and APT. Two isolated star markers on the far left vertical axis indicate the zero-shot TeCoA baseline performance. These markers are positioned significantly lower than the starting points of the proposed methods at one shot.}
\end{figure}

\noindent \textbf{Adversarial few-shot classification.}
In this scenario, we assess the capability of prompts to learn robust representations from limited labeled data.
Specifically, models are tuned using $ \{1, 4, 16\}$ shots per class and evaluated on the remaining samples.
The average performance on 11 datasets is reported in Figure \ref{fig:few_shot}, where `H' denotes the harmonic mean of the robustness and accuracy. 
ADAPT consistently outperforms TeCoA and APT across all settings.
Even in the extreme data-scarce regime (i.e., 1-shot), ADAPT achieves significant improvements, surpassing APT (TeCoA) by 0.90\% (1.9\%) and 1.64\% (2.7\%) in terms of the robustness and `H', respectively.
Notably, ADAPT$_{\text M}$ achieves the highest robustness and `H' among all baselines under all settings, validating that our proposed method effectively learns robust features without overfitting to shortcuts.\footnote{***In Figure 4, there is a method called D-ADAPT. revise the legend. \edit{***revised.}}\footnote{***why are the performance of TeCoA two points instead of two curves in Figure 4?\edit{***Because TeCoA is a zero-shot method.}}

\noindent \textbf{Adversarial robustness evaluation under various attacks.}
Here we evaluate the adversarial robustness of \edit{adversarial} prompt tuning\footnote{***adversarial prompt tuning? \edit{***revised}} methods using a wider variety of attacks, including Carlini \& Wagner (CW) attack \citep{cw} and AutoAttack (AA) \citep{autoattack}. CW represents a strong optimization-based attack aiming for minimal perturbations, while AA serves as a standardized, ensemble-based attack for reliable robustness assessment. 
Both CW and AA generate more potent adversarial examples than PGD, making them stronger attack methods for more rigorous evaluation. 
Table \ref{tab:aa} reports the performance across 10 datasets (excluding ImageNet)\footnote{***why excluding this dataset? any reason? or any reference to follow? \edit{***Evaluating on Imagenet on AA may take more than one month.}}, where `$\mathrm{H}_\mathrm{a}$' denotes the harmonic mean computed over $\mathrm{Acc.}_\mathrm{Base}$, $\mathrm{PGD}_\mathrm{Base}$, $\mathrm{Acc.}_\mathrm{New}$, $\mathrm{PGD}_\mathrm{New}$, CW, and AA on both base and new classes.
The proposed methods consistently outperform the baselines under all attacks.
Specifically, ADAPT surpasses APT by 5.23\% in `$\mathrm{H}_\mathrm{a}$'.
Furthermore, ADAPT$_{\text M}$ achieves the best overall performance, outperforming the previous SOTA method (i.e., FAP) by 1.73\% in terms of `$\mathrm{H}_\mathrm{a}$'.\footnote{***what is H$_a$? \edit{***revised.}}
This consistent superiority against such a diverse and aggressive attack suite provides strong evidence that the robustness learned by ADAPT is generalized, stemming from the effective disentanglement of invariant semantic features from attack-specific shortcuts.

\begin{table}[!t]
  \centering
  \caption{Average performance under the adversarial base-to-new generalization setting with various attacks. $\epsilon=4/255$.}
  \resizebox{0.9\linewidth}{!}{
  \setlength{\tabcolsep}{1mm}{
    \begin{tabular}{l|cc|ccccc}
    \toprule
    \multirow{3}[6]{*}{Method} & \multicolumn{7}{c}{\textbf{Average over 10 datasets}} \\
\cmidrule{2-8}          & \multicolumn{2}{c|}{\textbf{Base}} & \multicolumn{4}{c|}{\textbf{New}} & \multirow{2}[4]{*}{$\mathrm{H}_\mathrm{a}$} \\
\cmidrule{2-7}          & Acc.  & PGD   & Acc.  & PGD   & CW    & \multicolumn{1}{c|}{AA} &  \\
    \midrule
    TeCoA & 40.54  & 13.59  & \textbf{38.08} & 14.90  & 14.59  & 13.18  & 17.86  \\
    \midrule
    APT   & 54.57  & 26.15  & 26.63  & 9.83  & 9.45  & 8.56  & 14.33  \\
    \rowcolor{Gray} ADAPT & \textbf{55.76} & 26.26  & 35.03  & 14.41  & 13.64  & 12.59  & 19.56  \\
    \midrule
    FAP   & 54.07  & 26.70  & 36.36  & 14.98  & 13.32  & 12.19  & 19.52  \\
    \rowcolor{Gray} ADAPT$_{\text M}$ & 53.03  & \textbf{27.26} & 37.80  & \textbf{16.45} & \textbf{14.86} & \textbf{13.79} & \textbf{21.23} \\
    \bottomrule
    \end{tabular}}}
  \label{tab:aa}%
\end{table}%

\begin{table}[!t]
  \centering
  \caption{Ablation study of the proposed ADAPT framework. We report the average performance under the adversarial base-to-new generalization setting across 11 datasets. $\epsilon=4/255$.}
  \resizebox{0.9\linewidth}{!}{
  \setlength{\tabcolsep}{2.5mm}{
    \begin{tabular}{l|cc|cc|c}
    \toprule
    \multirow{2}[4]{*}{Methods} & \multicolumn{2}{c|}{Base} & \multicolumn{2}{c|}{New} & \multirow{2}[4]{*}{{$\mathrm{H}_\mathrm{b}$}} \\
\cmidrule{2-5}          & Acc   & Rob   & Acc   & Rob   &  \\
    \midrule
    TeCoA & 44.51  & 14.62  & 42.16  & 16.14  & 22.66  \\
    \midrule
    APT   & 58.62  & 27.40  & 30.13  & 10.96  & 22.47  \\
    \midrule
    \rowcolor{Gray} ADAPT & 60.04  & 27.48  & 38.98  & 15.65  & \textbf{28.05} \\
      w/o $\mathcal{P}_\mathrm{d}$ & 59.95  & 27.55  & 36.43  & 14.75  & 26.99  \\
      w/o $\mathcal{L}_\mathrm{dis}$ & 59.80  & 27.49  & 38.25  & 14.99  & 27.41  \\
      w/o $\mathcal{L}_{\mathrm{sem}}$ & 58.34  & 27.37  & 34.14  & 12.54  & 24.58  \\
      w/o $\mathcal{L}_{\mathrm{orth}}$ & 60.19  & 27.56  & 37.52  & 13.16  & 25.72  \\
    \bottomrule
    \end{tabular}}}
  \label{tab:ablation}%
\end{table}%

\emph{Due to page limit, more experimental results can be found in Appendix \ref{appendix:Additional_Experimental_Results}, including generalization to alternative VLM backbones, computational cost analysis, and the results under the settings of adversarial cross-dataset generalization and adversarial domain generalization.}

\begin{figure*}[t]
  \centering
  \begin{subfigure}{0.22\linewidth}
    \includegraphics[width=\linewidth]{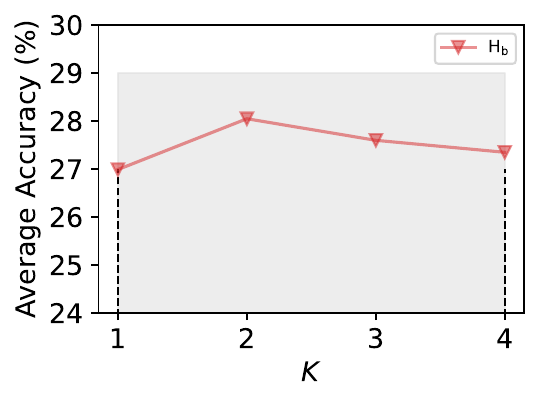}
    \caption{}
    \label{K}
  \end{subfigure}
  \hfill
  \begin{subfigure}{0.22\linewidth}
    \includegraphics[width=\linewidth]{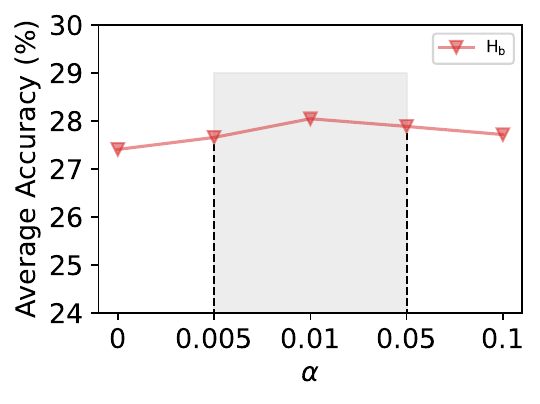}
    \caption{}
    \label{alpha}
  \end{subfigure}
  \hfill
  \begin{subfigure}{0.22\linewidth}
    \includegraphics[width=\linewidth]{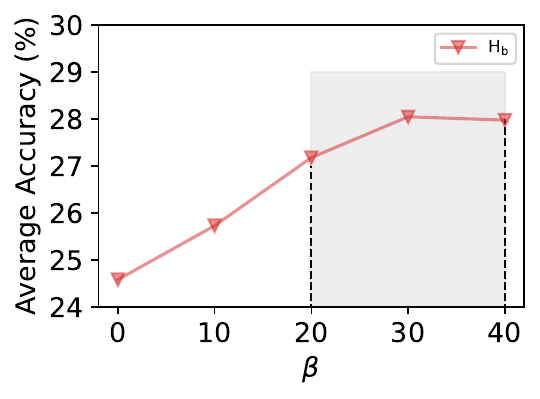}
    \caption{}
    \label{beta}
  \end{subfigure}
  \hfill
  \begin{subfigure}{0.22\linewidth}
    \includegraphics[width=\linewidth]{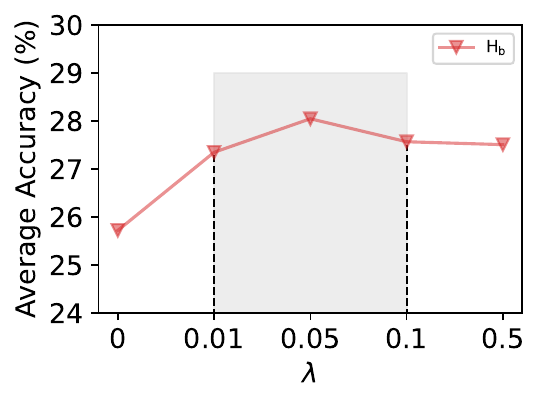}
    \caption{}
    \label{lambda}
  \end{subfigure}
  \caption{The average performance of ADAPT on all datasets w.r.t. hyperparameters (i.e., $K$, $\alpha$, $\beta$, and $\lambda$) under the adversarial base-to-new generalization setting.}
  \label{fig_parameter}
  \Description{A series of four line charts analyzes the impact of adjusting four hyperparameters on average accuracy. Chart A plots parameter K on a horizontal axis from 1 to 4 and shows accuracy peaking at K equals 2. Chart B plots parameter alpha from 0 to 0.1 with a gentle peak around 0.01. Chart C plots parameter beta from 0 to 40, where accuracy rises sharply from 0 to 30 before leveling off. Chart D plots parameter lambda from 0 to 0.5 and reaches a distinct peak at 0.05. Each subchart includes a light gray rectangular shaded background that visually highlights the parameter range producing the highest and most stable accuracy results.}
\end{figure*}

\subsection{Ablation Studies}
\label{ablation}
In this section, we conduct ablations to analyze the contributions of the pool of decoy prompts $\mathcal{P}_\mathrm{d}$, the dissimilarity loss $\mathcal{L}_\mathrm{dis}$, the semantic loss $\mathcal{L}_{\mathrm{sem}}$, and the orthogonal loss $\mathcal{L}_{\mathrm{orth}}$.

As shown in Table \ref{tab:ablation}, the ADAPT method achieves good performance with an average `$\mathrm{H}_\mathrm{b}$' of 28.05$\%$, verifying the necessity of integrating these components.
We observe that removing $\mathcal{L}_{\mathrm{sem}}$ results in the most severe performance degradation. This drop is expected, as $\mathcal{L}_{\mathrm{sem}}$ serves as a semantic regularizer that tethers the target prompt to the general semantic space. Without it, the model may suffer from catastrophic forgetting of its pre-trained knowledge.
Furthermore, the exclusion of the orthogonal loss $\mathcal{L}_{\mathrm{orth}}$ leads to a drop of 2.33\% in `$\mathrm{H}_\mathrm{b}$'.
While $\mathcal{L}_{\mathrm{sem}}$ ensures the model remains a competent VLM, $\mathcal{L}_{\mathrm{orth}}$ is responsible for explicitly enforcing independence between the target and decoy prompts to actively filter out pseudo-robust features.
Additionally, removing the dissimilarity loss $\mathcal{L}_\mathrm{dis}$ or the pool of decoy prompts $\mathcal{P}_\mathrm{d}$ also impairs the generalization capability on new classes, confirming that a diverse set of decoy prompts explicitly pushed away from valid semantics is a requisite to entrap diverse pseudo-robust features effectively.

\section{Analysis on Hyperparameter Sensitivity}
\label{appendix:ablation}

We analyze the sensitivity of $K$, $\alpha$, $\beta$, and $\lambda$ using the average performance over 11 datasets under adversarial base-to-new generalization with $\epsilon=4/255$.

\noindent \textbf{Effect of $K$.}
As shown in Figure~\ref{K}, increasing $K$ from 1 to 2 improves performance, indicating that multiple decoy prompts better capture diverse pseudo-robust features. Further increasing $K$ provides no additional benefit and slightly complicates optimization.

\noindent \textbf{Effect of $\alpha$.}
The coefficient $\alpha$ controls the dissimilarity loss $\mathcal{L}_\mathrm{dis}$ in Eq.~\eqref{eq_decoy_total}. Figure~\ref{alpha} shows that ADAPT is relatively insensitive to $\alpha$ within $[0.005,0.1]$, with the best performance at $\alpha=0.01$.

\noindent \textbf{Effect of $\beta$.}
The coefficient $\beta$ weights the semantic loss $\mathcal{L}_\mathrm{sem}$ in Eq.~\eqref{eq_target_total}. As shown in Figure~\ref{beta}, small values of $\beta$ lead to substantial degradation, while the performance stabilizes for $\beta\in[30,40]$, highlighting the importance of preserving general semantic knowledge.

\noindent \textbf{Effect of $\lambda$.}
The coefficient $\lambda$ controls the orthogonal loss $\mathcal{L}_\mathrm{orth}$ in Eq.~\eqref{eq_target_total}. As shown in Figure~\ref{lambda}, the performance improves substantially as $\lambda$ increases from 0 to 0.05, confirming the benefit of explicitly disentangling the target and decoy prompts. The best performance is achieved at $\lambda=0.05$, while larger values lead to only a slight decline and remain relatively stable, indicating that ADAPT is not overly sensitive to $\lambda$ within a moderate range.

\begin{figure}[!htbp]
  \centering
  \begin{subfigure}{0.46\linewidth}
    \includegraphics[width=\linewidth]{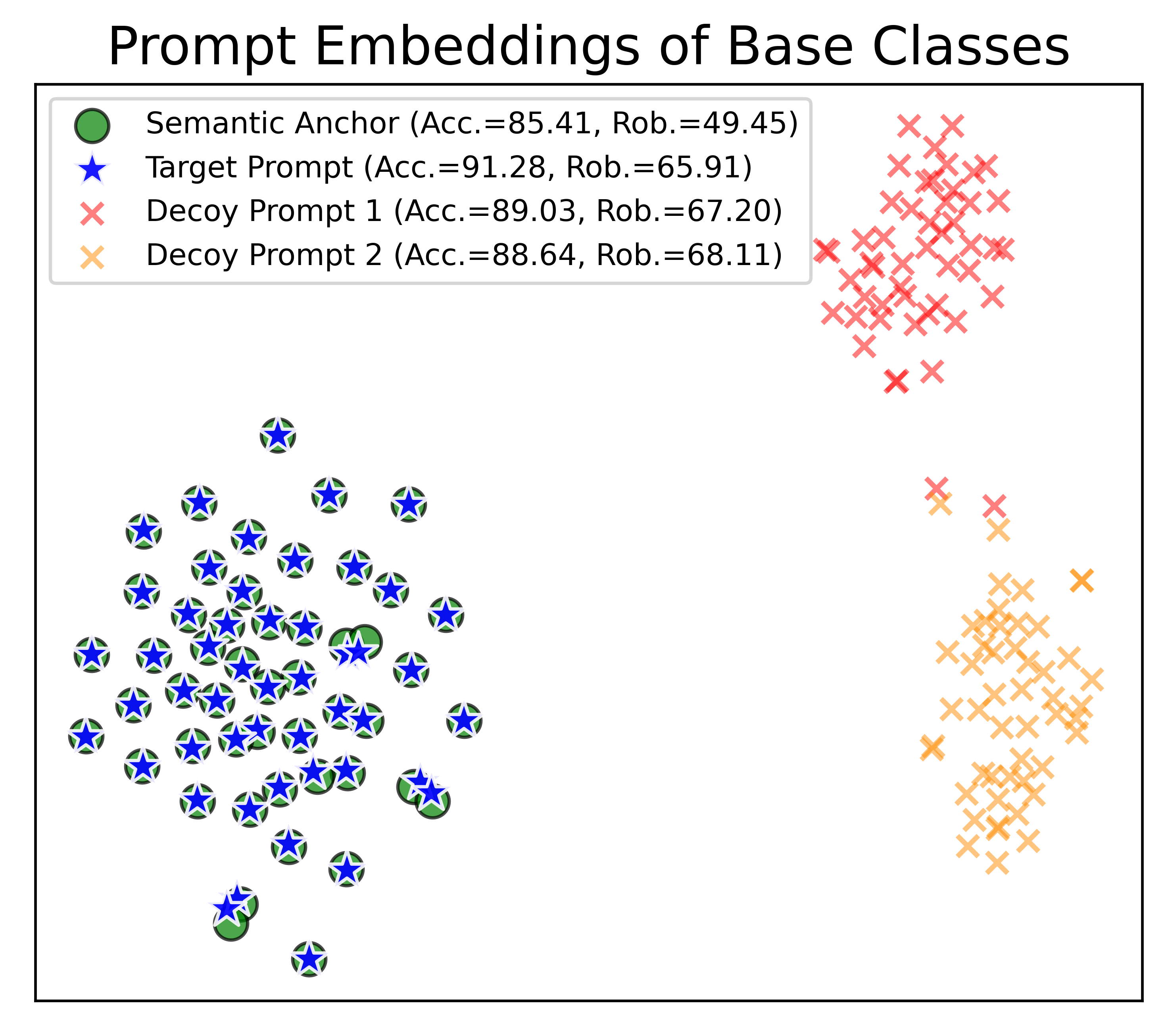}
  \end{subfigure}
\hfill
  \begin{subfigure}{0.46\linewidth}
    \includegraphics[width=\linewidth]{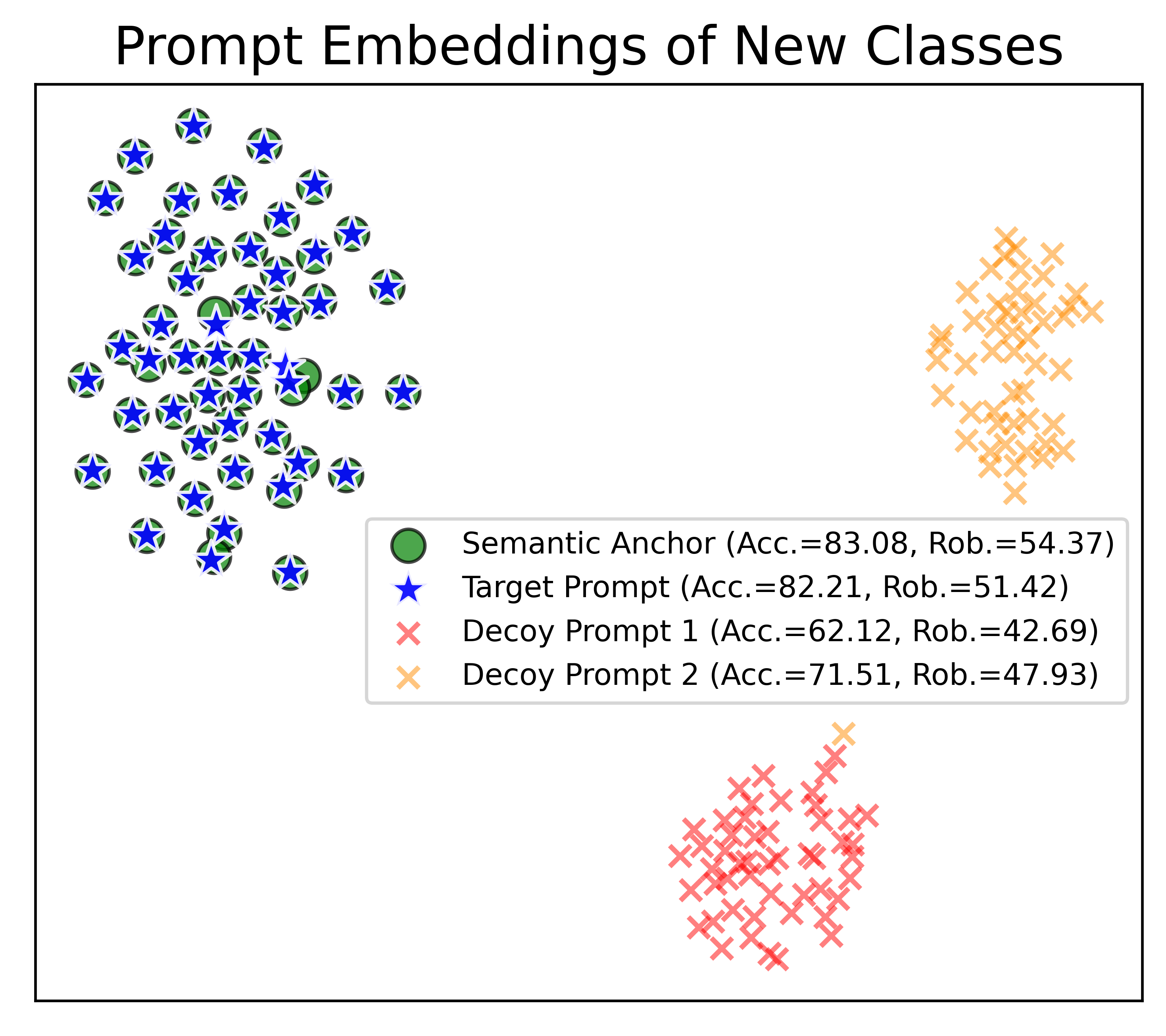}
  \end{subfigure}
    \caption{t-SNE visualization of prompt embeddings on Caltech101. The target prompts cluster around general semantic anchors, while the decoy prompts lie in a separate subspace.}
    \label{fig:tsne}
    \Description{Two side-by-side scatter plots display the 2D spatial clustering of prompt embeddings for base and new classes. The left chart for base classes and the right chart for new classes exhibit identical spatial layout patterns. On the left side of each chart, the target prompt markers cluster tightly and overlap directly with markers representing general semantic anchors. On the far right side of each chart, two distinct sets of decoy prompts form their own independent clusters. These decoy clusters are spatially isolated from the semantic anchors on the left. The two decoy clusters also maintain a large physical distance from each other, indicating they have converged into separate sub-spaces without interference.}
\end{figure}

\section{Visualization of Prompt Embeddings}
To verify the effectiveness of the proposed dual-prompt mechanism, we visualize the prompt embeddings using t-SNE in Figure~\ref{fig:tsne}. The results reveal a clear disentanglement in the latent space. Target prompts (in blue) consistently align with general semantic anchors (in green) across both base and new classes, confirming that our semantic loss effectively enforces the learning of semantically invariant features. Crucially, the decoy prompts (in red) diverge into distinct non-semantic subspaces isolated from the anchors and, notably, separated from each other. Despite this deviation, both groups of decoy prompts achieve good performance on base classes but degrade significantly on new classes.
This contrast indicates that different decoy prompts converge to diverse class-specific shortcuts. Though these features are discriminative on the training classes, they lack transferability. By employing a pool of decoy prompts to act as ``diverse traps'' for these varying shortcuts, ADAPT allows the target prompt to learn purified robust features.

\section{Conclusion}
In this paper, we identify and formally define the phenomenon of ``robust generalization overfitting'' in adversarial prompt tuning, revealing that the failure of robustness generalization stems from the model's reliance on pseudo-robust features.
To address this, we propose \textbf{ADAPT}, which is guided by the philosophy of ``Learning What Not to Learn''. Specifically, ADAPT employs a dual-prompt mechanism with explicit orthogonal loss, successfully entrapping non-generalizable pseudo-robust features into decoy prompts while guiding the target prompt to learn robust features.
Experiments on benchmark datasets demonstrate that ADAPT achieves state-of-the-art performance, particularly in the challenging base-to-new generalization settings.

\begin{acks}
This work was supported by National Natural Science Foundation of China under Grant no. 62136005, Shenzhen fundamental research program JCYJ20250604144724032, and the China Postdoctoral Science Foundation under Grant No. 2026M791631.
\end{acks}

\bibliographystyle{ACM-Reference-Format}
\balance
\bibliography{references}


\clearpage
\appendix
\renewcommand{\thetable}{A\arabic{table}}
\renewcommand{\thefigure}{A\arabic{figure}}
\setcounter{table}{0}
\setcounter{figure}{0}

\noindent \textbf{\large{Contents of the Appendix}}
\begin{enumerate}
    \item Appendix \ref{app:theory} — Theoretical Analysis
    \item Appendix \ref{appendix:Implementation_Details} — Setting Details
    \item Appendix \ref{appendix:Additional_Experimental_Results} — Additional Experimental Results
        \begin{itemize}
            \item \ref{appendix:robust_backbones} - Generalization to Alternative VLM Backbones
            \item \ref{appendix:b2n} — Comprehensive Results of Base-to-new Generalization
            \item \ref{appendix:fewshot} — Full Results of Adversarial Few-shot Classification
            \item \ref{appendix:dg} — Results of Adversarial Domain Generalization
            \item \ref{appendix:xd} — Results of Adversarial Cross-dataset Generalization
            \item \ref{appendix:cost} — Computational Cost Analysis
            \item \ref{app:projection_analysis} — Empirical Validation of the Semantic--Shortcut Decomposition
            \item \ref{app:rgo_settings} — Robust Generalization Overfitting under Different Training Settings
            \item \ref{app:dgaa} — Decoy-Guided Adaptive Attack
            
        \end{itemize}
\end{enumerate}

\section{Theoretical Analysis}\label{app:theory}

\subsection{Assumptions and Main Results (restated)}

\paragraph{Notation.}

All CLIP image/text embeddings are $\ell_2$-normalized unless stated otherwise.

\begin{assumption}
\label{ass:shortcut_shift}
${z}(\mathbf{x})\in\mathbb{R}^d$ is the CLIP image feature of a (possibly adversarial)
input $\mathbf{x}$, $\|{z}(\mathbf{x})\|_2=1$.
$\{{\mathbf{t}}_{\mathrm{H}}^{c}\}_{c=1}^{C}$ is the hand-crafted text embeddings (semantic anchors), $\|{\mathbf{t}}_{\mathrm{H}}^{c}\|_2=1$ for all $c$.

Assume there exists an orthogonal decomposition of the feature space
$\mathbb{R}^{d}=\mathcal{S}\oplus\mathcal{U}$,
where $\mathcal{S}$ is a semantic subspace and $\mathcal{U}$ is a pseudo-robust subspace, such that:
\begin{enumerate}
\item \textbf{(Feature decomposition)} For any (possibly adversarially perturbed) input $\mathbf{x}$ from either base or new classes,
\begin{equation}
{z}(\mathbf{x})=s(\mathbf{x})+u(\mathbf{x}),\quad
s(\mathbf{x})\in\mathcal{S},\ u(\mathbf{x})\in\mathcal{U},\quad
\|u(\mathbf{x})\|_2\le B.
\end{equation}
The distribution of $u(\mathbf{x})$ may change arbitrarily from base to new classes.

\item \textbf{(Anchor semantics)} ${\mathbf{t}}_{\mathrm{H}}^{c}\in\mathcal{S}$ for all $c\in[C]$.

\item \textbf{(New-class anchor margin)} There exist $\gamma>0$ and $\zeta\in[0,1]$ such that with probability at least $1-\zeta$
over a new-class sample $(\mathbf{x},y)$,
\begin{equation}
m_{\mathrm{H}}(\mathbf{x},y)
:=\min_{c\ne y}\ \langle s(\mathbf{x}),\,{\mathbf{t}}_{\mathrm{H}}^{y}-{\mathbf{t}}_{\mathrm{H}}^{c}\rangle
\ \ge\ \gamma.
\end{equation}
where $< \cdot ,\cdot >$ denotes the inner product operation.
\end{enumerate}
\end{assumption}


\begin{theorem}[Restatement of Theorem~3.2]\label{thm:adapt_main_restate}
$\{{\mathbf{t}}_{\mathrm{t}}^{c}\}_{c=1}^{C}$ is the text embeddings of the target prompts, and define
the target classifier
\begin{equation}
\hat{y}(\mathbf{x}) := \arg\max_{c\in C} \ \langle {z}(\mathbf{x}), {\mathbf{t}}_{\mathrm{t}}^{c}\rangle.
\end{equation}
Define
\begin{equation}
\varepsilon_{\mathrm{sem}}
:=\max_{c\in C}\|{\mathbf{t}}_{\mathrm{t}}^{c}-{\mathbf{t}}_{\mathrm{H}}^{c}\|_2,
\qquad
\varepsilon_{\mathcal{U}}
:=\max_{c\in C}\|P_{\mathcal{U}}{\mathbf{t}}_{\mathrm{t}}^{c}\|_2,
\end{equation}
where $P_{\mathcal{U}}(\cdot)$ is the orthogonal projector onto $\mathcal{U}$.
Under Assumption~\ref{ass:shortcut_shift}, if
\begin{equation}
\gamma > 2\varepsilon_{\mathrm{sem}} + 2B\,\varepsilon_{\mathcal{U}},
\end{equation}
then for any new-class sample $(\mathbf{x},y)$ satisfying $m_{\mathrm{H}}(\mathbf{x},y)\ge\gamma$,
the classifier predicts correctly, i.e., $\hat{y}(\mathbf{x})=y$,
regardless of the shortcut component $u(\mathbf{x})$ as long as $\|u(\mathbf{x})\|_2\le B$.
Consequently, the testing error rate of $\hat{y}$ on new classes is at most $\zeta$.
\end{theorem}

\subsection{Proof of Theorem~\ref{thm:adapt_main_restate}}

\begin{proof}
Fix a new-class sample $(\mathbf{x},y)$ such that $m_{\mathrm{H}}(\mathbf{x},y)\ge \gamma$.
For any competing class $c\neq y$, consider the score difference
\begin{equation}
\Delta_{y,c}(\mathbf{x})
:=
\langle {z}(\mathbf{x}), {\mathbf{t}}_{\mathrm{t}}^{y}-{\mathbf{t}}_{\mathrm{t}}^{c}\rangle.
\end{equation}
If $\Delta_{y,c}(\mathbf{x})>0$ holds for all $c\neq y$, then
$\langle {z}(\mathbf{x}), {\mathbf{t}}_{\mathrm{t}}^{y}\rangle >
\langle {z}(\mathbf{x}), {\mathbf{t}}_{\mathrm{t}}^{c}\rangle$
for all $c\neq y$, implying $\hat{y}(\mathbf{x})=y$.

\paragraph{Step 1: Decompose the score difference into semantic and shortcut parts.}
By Assumption~\ref{ass:shortcut_shift}(1),
${z}(\mathbf{x})=s(\mathbf{x})+u(\mathbf{x})$ with
$s(\mathbf{x})\in\mathcal{S}$ and $u(\mathbf{x})\in\mathcal{U}$.
Thus,
\begin{equation}
\Delta_{y,c}(\mathbf{x})
=
\underbrace{\langle s(\mathbf{x}), {\mathbf{t}}_{\mathrm{t}}^{y}-{\mathbf{t}}_{\mathrm{t}}^{c}\rangle}_{\text{semantic term}}
+
\underbrace{\langle u(\mathbf{x}), {\mathbf{t}}_{\mathrm{t}}^{y}-{\mathbf{t}}_{\mathrm{t}}^{c}\rangle}_{\text{shortcut term}}.
\end{equation}

\paragraph{Step 2: Lower bound the semantic term using anchor margin and $\varepsilon_{\mathrm{sem}}$.}
Write ${\mathbf{t}}_{\mathrm{t}}^{k}={\mathbf{t}}_{\mathrm{H}}^{k}+\boldsymbol{\Delta}_{k}$ for $k\in\{y,c\}$.
By definition of $\varepsilon_{\mathrm{sem}}$, we have $\|\boldsymbol{\Delta}_{k}\|_2\le \varepsilon_{\mathrm{sem}}$.
Then
\begin{align}
\langle s(\mathbf{x}), {\mathbf{t}}_{\mathrm{t}}^{y}-{\mathbf{t}}_{\mathrm{t}}^{c}\rangle
&=
\langle s(\mathbf{x}), {\mathbf{t}}_{\mathrm{H}}^{y}-{\mathbf{t}}_{\mathrm{H}}^{c}\rangle
+
\langle s(\mathbf{x}), \boldsymbol{\Delta}_{y}-\boldsymbol{\Delta}_{c}\rangle \\
&\ge
\gamma
-
\|s(\mathbf{x})\|_2\big(\|\boldsymbol{\Delta}_{y}\|_2+\|\boldsymbol{\Delta}_{c}\|_2\big),
\end{align}
where the inequality uses $m_{\mathrm{H}}(\mathbf{x},y)\ge \gamma$.
Moreover, since $\|{z}(\mathbf{x})\|_2=1$ and $s(\mathbf{x})$ is an orthogonal component,
$\|s(\mathbf{x})\|_2\le 1$.
Therefore,
\begin{equation}
\langle s(\mathbf{x}), {\mathbf{t}}_{\mathrm{t}}^{y}-{\mathbf{t}}_{\mathrm{t}}^{c}\rangle
\ge
\gamma - 2\varepsilon_{\mathrm{sem}}.
\end{equation}

\paragraph{Step 3: Lower bound the shortcut term using $\varepsilon_{\mathcal{U}}$.}
Since $u(\mathbf{x})\in\mathcal{U}$,
\begin{equation}
\langle u(\mathbf{x}), {\mathbf{t}}_{\mathrm{t}}^{y}-{\mathbf{t}}_{\mathrm{t}}^{c}\rangle
=
\langle u(\mathbf{x}), P_{\mathcal{U}}({\mathbf{t}}_{\mathrm{t}}^{y}-{\mathbf{t}}_{\mathrm{t}}^{c})\rangle.
\end{equation}
By Cauchy--Schwarz and Assumption~\ref{ass:shortcut_shift}(1), 
\begin{align}
\langle u(\mathbf{x}), {\mathbf{t}}_{\mathrm{t}}^{y}-{\mathbf{t}}_{\mathrm{t}}^{c}\rangle
&\ge
-\|u(\mathbf{x})\|_2\cdot
\|P_{\mathcal{U}}{\mathbf{t}}_{\mathrm{t}}^{y}-P_{\mathcal{U}}{\mathbf{t}}_{\mathrm{t}}^{c}\|_2 \\
&\ge
-B\big(\|P_{\mathcal{U}}{\mathbf{t}}_{\mathrm{t}}^{y}\|_2+\|P_{\mathcal{U}}{\mathbf{t}}_{\mathrm{t}}^{c}\|_2\big)
\ \ge\ -2B\,\varepsilon_{\mathcal{U}},
\end{align}
where the last inequality uses the definition of $\varepsilon_{\mathcal{U}}$.

\paragraph{Step 4: Combine the bounds.}
Combining Steps 2--3 yields, for any $c\neq y$,
\begin{equation}
\Delta_{y,c}(\mathbf{x})
\ge
\gamma - 2\varepsilon_{\mathrm{sem}} - 2B\,\varepsilon_{\mathcal{U}}.
\end{equation}
Under the condition $\gamma > 2\varepsilon_{\mathrm{sem}} + 2B\,\varepsilon_{\mathcal{U}}$,
we have $\Delta_{y,c}(\mathbf{x})>0$ for all $c\neq y$, hence $\hat{y}(\mathbf{x})=y$.

\paragraph{Step 5: Error bound on new classes.}
By Assumption~\ref{ass:shortcut_shift}(3), the event $\{m_{\mathrm{H}}(\mathbf{x},y)\ge\gamma\}$ holds with probability at least $1-\zeta$.
Since we just showed the classifier is correct whenever this event holds,
\begin{equation}
\Pr_{(\mathbf{x},y)\sim \mathcal{D}_{\mathrm{new}}}\big[\hat{y}(\mathbf{x})\neq y\big]
\le
\Pr\big[m_{\mathrm{H}}(\mathbf{x},y)<\gamma\big]
\le \zeta.
\end{equation}
This also implies invariance to any base-to-new shift that changes only $u(\mathbf{x})$ (within $\|u(\mathbf{x})\|_2\le B$)
while keeping $s(\mathbf{x})$ unchanged, because the above lower bound does not depend on the particular realization of $u(\mathbf{x})$
beyond its norm bound.
\end{proof}

\section{Setting Details}
\label{appendix:Implementation_Details}

\noindent \textbf{Datasets.}
Following the protocols established in \citep{coop,fap,apt}, we conduct experiments under four distinct settings: adversarial base-to-new generalization, adversarial few-shot classification, adversarial cross-dataset generalization, and adversarial domain generalization.
The first three settings are evaluated on a diverse suite of 11 image classification datasets, covering a wide spectrum of visual recognition tasks.
Specifically, these include ImageNet \citep{imagenet} and Caltech101 \citep{caltech101} for general object recognition; OxfordPets \citep{oxfordpets}, StanfordCars \citep{stanfordcars}, Flowers102 \citep{oxfordflowers}, Food101 \citep{food}, and FGVCAircraft \citep{FGVCAircraft} for fine-grained visual categorization; SUN397 \citep{sun397}, DTD \citep{dtd}, and EuroSAT \citep{eurosat} for scene, texture, and satellite imagery classification, respectively; UCF101 \citep{ucf101} for action recognition.
For the setting of adversarial domain generalization, we train on the ImageNet dataset and evaluate on its out-of-distribution (OOD) variants, including ImageNetV2 \citep{imagenet_v2}, ImageNet-Sketch \citep{imagenet_sketch}, ImageNet-A \citep{imagenet_a}, and ImageNet-R \citep{imagenet_r}.

\noindent \textbf{Implementation Details.}
Our implementation is built upon the codebases of CoOp \citep{coop} and APT \citep{apt}. All experiments are conducted on NVIDIA GeForce RTX 3090, except for the ImageNet dataset, which is on NVIDIA A100. To ensure a fair comparison with prior works, the backbone architecture follows the default robust configuration \citep{tecoa} to align with the adversarial prompt tuning benchmark \citep{apt}. All experiments are conducted using the ViT-B/32 CLIP model. We strictly adhere to the experimental settings specified in the original implementations of all baseline methods, including training epochs, learning rate schedules, and data augmentation strategies.
Training is conducted using SGD. The learning rate is decayed using the cosine annealing rule. The maximum epoch for ADAPT and APT is set to 200, 100, and 50 for 16, 4, and 1 shots, respectively. For ImageNet, they are 50, 20, and 20. As for ADAPT$_{\text M}$ and FAP, the maximum epoch is fixed to 10. A warm-up strategy is used by fixing the learning rate to $10^{-5}$ during the first epoch. 
For the APT and ADAPT methods, the length of the learnable prompt vectors is fixed to 16 tokens.
For the FAP and ADAPT$_{\text M}$ methods, we use a deep prompting strategy, where prompt vectors of length 2 are inserted into both the vision and text branches across the first 9 transformer blocks.
For the ADAPT method, $\alpha$, $\beta$, and $\lambda$ are set to 0.01, 30\footnote{***so large}, and 0.05, respectively.
The size of the pool of decoy prompts $K$ is set to 2 for ADAPT and 4 for ADAPT$_{\text M}$, respectively.
To encourage the decoy prompt to converge to shortcuts rapidly, we set a higher learning rate for the decoy prompt (i.e., 0.01), and the learning rate for the target prompt is set to 0.002 for ADAPT and 0.0035 for ADAPT$_{\text M}$.

For AutoAttack, we use the standard AutoAttack setting, which executes a suite of four attacks: APGD-CE (untargeted), APGD-T (targeted), FAB-T (targeted), and Square (black-box). All component attacks utilize the default setting of 1 restart.

The hand-crafted prompts for different datasets follow \citet{clip, coop} and are shown below:
\begin{quote}
\begin{small}\begin{verbatim}
ImageNet: "a photo of a [CLS]."
Caltech101: "a photo of a [CLS]." 
OxfordPets: "a photo of a [CLS], a type of pet." 
StanfordCars: "a photo of a [CLS]." 
OxfordFlowers: "a photo of a [CLS], a type of flower." 
Food101: "a photo of [CLS], a type of food." 
FGVCAircraft: "a photo of a [CLS], a type of aircraft." 
SUN397: "a photo of a [CLS]." 
DTD: "a photo of a [CLS], a type of texture." 
EuroSAT: "a centered satellite photo of [CLS]." 
UCF101: "a photo of a person doing [CLS]." 
\end{verbatim}\end{small}
\end{quote}
Note that [CLS] denotes the placeholder for the class name.

\section{Additional Experimental Results}
\label{appendix:Additional_Experimental_Results}

\subsection{Generalization to Alternative VLM Backbones}
\label{appendix:robust_backbones}
To further demonstrate that ADAPT is model-agnostic, we extend our experiments to three alternative vision-language models: a larger CLIP model (ViT-L/14), a more recent and powerful model pre-trained via the FARE method \citep{fare}, and OpenCLIP. \cref{table_vlm} shows the average adversarial base-to-new generalization performance across 10 datasets (excluding ImageNet) under a perturbation budget of $\epsilon=4/255$.

First, to assess the scalability, we apply our method to the larger TeCoA model with a ViT-L/14 backbone. As shown in \cref{tab:vitl}, APT continues to suffer from robust generalization overfitting and yields a new-class robustness of only 32.47\%. In contrast, ADAPT successfully mitigates this degradation by improving the new-class robustness to 38.01\% and elevating the overall harmonic mean `$\mathrm{H}_\mathrm{b}$' from 48.18\% to 52.87\%. Furthermore, our multi-modal extension ADAPT$_{\text M}$ achieves an `$\mathrm{H}_\mathrm{b}$' of 53.94\%, consistently outperforming the competitive FAP baseline, which scores 52.38\%.

We observe a similar performance trajectory when adapting these methods to the FARE backbone. As shown in \cref{tab:vit32_fare}, the baseline APT method experiences a severe drop in new-class robustness to 13.49\%, resulting in a suboptimal `$\mathrm{H}_\mathrm{b}$' of 28.25\%. ADAPT effectively addresses this vulnerability by increasing the new-class robustness to 20.31\% and achieving an `$\mathrm{H}_\mathrm{b}$' of 35.54\%. Within the multi-modal prompting paradigm, ADAPT$_{\text M}$ attains the highest `$\mathrm{H}_\mathrm{b}$' of 39.24\% and surpasses the strongest baseline (FAP at 36.65\%). 

We further evaluate ADAPT using OpenCLIP ViT-B/32 pretrained on LAION2B.
Table~\ref{tab:openclip} reports the average performance over 10 datasets
under $\epsilon=4/255$. ADAPT consistently improves over APT, increasing
new-class accuracy and robustness from 39.11\% / 13.69\% to
47.13\% / 19.09\%, respectively. These results show that the robust
generalization gains of ADAPT transfer to CLIP-like models trained with
different data and pretraining pipelines.

These consistent improvements across different VLMs confirm that standard adversarial prompt tuning inherently struggles with pseudo-robust features. More importantly, those results demonstrate that ADAPT is a broadly applicable framework and can successfully enhance the robustness of different VLMs.

\begin{table*}[!t]
\renewcommand{\arraystretch}{1}
\centering
\caption{Performance under the setting of adversarial base-to-new generalization on different VLM backbones. `hc' denotes hand-crafted prompt, `tp' denotes textual prompt tuning, and `mmp' denotes multi-modal prompt tuning.}
\captionsetup[sub]{skip=1pt}
\subfloat[ \label{tab:vitl}]{
    \begin{minipage}[t]{0.45\textwidth}
    \scalebox{0.8}{
    \setlength{\tabcolsep}{1.5mm}{
    \begin{tabular}{c|l|cc|cc|c}
    \toprule
    \multirow{3}[6]{*}{$\epsilon=4/255$} & \multicolumn{1}{c|}{\multirow{3}[6]{*}{Methods}} & \multicolumn{5}{c}{Average} \\
\cmidrule{3-7}          &       & \multicolumn{2}{c|}{Base} & \multicolumn{2}{c|}{New} & \multirow{2}[4]{*}{$\mathrm{H}_\mathrm{b}$} \\
\cmidrule{3-6}          &       & Acc.  & Rob.  & Acc.  & Rob.  &  \\
    \midrule
    hc    & \textbf{TeCoA (ViT-L/14)} & 49.30  & 35.48  & 54.09  & 40.04  & 43.51  \\
    \midrule
    \multirow{2}[1]{*}{tp} & +APT  & 75.98  & 59.94  & 44.68  & 32.47  & 48.18  \\
          & \cellcolor{Gray}+ADAPT & \cellcolor{Gray}\textbf{76.28} & \cellcolor{Gray}\textbf{61.22} & \cellcolor{Gray}50.26  & \cellcolor{Gray}38.01  & \cellcolor{Gray}\textbf{52.87}  \\
    \midrule
    \multirow{2}[1]{*}{mmp} & +FAP  & 72.21  & 58.66  & 50.81  & 38.78  & 52.38  \\
          & \cellcolor{Gray}+ADAPT$_{\text M}$ & \cellcolor{Gray}72.01  & \cellcolor{Gray}57.70  & \cellcolor{Gray}\textbf{54.21} & \cellcolor{Gray}\textbf{40.83} & \cellcolor{Gray}\textbf{53.94} \\
    \bottomrule
    \end{tabular}}}
\end{minipage}
}
\hfill
\subfloat[\label{tab:vit32_fare}]{
    \begin{minipage}[t]{0.45\textwidth}
    \scalebox{0.8}{
    \setlength{\tabcolsep}{1.5mm}{
    \begin{tabular}{c|l|cc|cc|c}
    \toprule
    \multirow{3}[6]{*}{$\epsilon=4/255$} & \multicolumn{1}{c|}{\multirow{3}[6]{*}{Methods}} & \multicolumn{5}{c}{Average} \\
\cmidrule{3-7}          &       & \multicolumn{2}{c|}{Base} & \multicolumn{2}{c|}{New} & \multirow{2}[4]{*}{$\mathrm{H}_\mathrm{b}$} \\
\cmidrule{3-6}          &       & Acc.  & Rob.  & Acc.  & Rob.  &  \\
    \midrule
    hc    & \textbf{FARE (ViT-B/32)} & 50.69  & 18.84  & 51.84  & 20.95  & 28.60  \\
    \midrule
    \multirow{2}[2]{*}{tp} & +APT  & 69.25  & 35.54  & 40.20  & 13.49  & 28.25  \\
          & \cellcolor{Gray}+ADAPT & \cellcolor{Gray}70.30  & \cellcolor{Gray}35.21  & \cellcolor{Gray}48.32  & \cellcolor{Gray}20.31  & \cellcolor{Gray}\textbf{35.54}  \\
    \midrule
    \multirow{2}[2]{*}{mmp} & +FAP  & 64.83  & 37.12  & 45.36  & 22.35  & 36.65  \\
          & \cellcolor{Gray}+ADAPT$_{\text M}$ & \cellcolor{Gray}64.11  & \cellcolor{Gray}36.17  & \cellcolor{Gray}50.41  & \cellcolor{Gray}25.74  & \cellcolor{Gray}\textbf{39.24} \\
    \bottomrule
    \end{tabular}}}
    \end{minipage}
}

\label{table_vlm}
\end{table*}

\begin{table}[htbp]
    \centering
    \caption{Adversarial base-to-new generalization with OpenCLIP ViT-B/32
    pretrained on LAION2B. Results are averaged over 10 datasets with
    $\epsilon=4/255$.}
    \setlength{\tabcolsep}{4pt}
    \begin{tabular}{c|cc|cc}
        \toprule
        \multirow{2}{*}{Method}
        & \multicolumn{2}{c|}{Base}
        & \multicolumn{2}{c}{New} \\
        \cmidrule{2-5}
        & Acc. & Rob. & Acc. & Rob. \\
        \midrule
        APT   & 69.78 & 35.40 & 39.11 & 13.69 \\
        ADAPT & \textbf{71.79} & \textbf{35.57}
              & \textbf{47.13} & \textbf{19.09} \\
        \bottomrule
    \end{tabular}
    \label{tab:openclip}
\end{table}

\subsection{Comprehensive Results of Adversarial Base-to-new Generalization}
\label{appendix:b2n}
In this section, we provide the adversarial base-to-new generalization results under the perturbation budget of $\epsilon=1/255$. 
The results in Table~\ref{appendix:table_base2new} show that, consistent with the findings under the larger perturbation budget ($\epsilon=4/255$), our methods demonstrate good robustness transferability.
APT suffers from significant overfitting to seen classes, exhibiting a large performance drop on new classes. 
In contrast, ADAPT effectively mitigates this issue. On average across 11 datasets, ADAPT outperforms APT by substantial margins of 7.06\% in accuracy and 6.35\% in robustness on new classes.
When extended to the multi-modal prompting, ADAPT$_{\text M}$ consistently surpasses the previous state-of-the-art method FAP. Specifically, ADAPT$_{\text M}$ achieves the highest average performance on new classes with 61.21\% in accuracy and 45.89\% in robustness, outperforming FAP by 4.59\% and 2.55\%, respectively.
Those results show that the disentanglement of robust and pseudo-robust features proposed in ADAPT is effective.

\begin{table*}[!t]
\renewcommand{\arraystretch}{1}
\centering
\caption{Performance of various methods under the adversarial base-to-new generalization setting. All methods (except TeCoA) are trained with 16 instances per base class. `$\mathrm{H}_\mathrm{b}$' denotes the harmonic mean accuracy between $\mathrm{Acc.}_\mathrm{Base}$, $\mathrm{Rob.}_\mathrm{Base}$, $\mathrm{Acc.}_\mathrm{New}$, and $\mathrm{Rob.}_\mathrm{New}$. $\epsilon=1/255$.}

\subfloat[\textbf{Average over 11 datasets}]{
    \begin{minipage}[t]{0.32\textwidth}
    \scalebox{0.72}{
    \setlength{\tabcolsep}{1.5mm}{
    \begin{tabular}{l|cc|cc|c}
    \toprule
    \multicolumn{1}{c|}{\multirow{2}[4]{*}{Methods}} & \multicolumn{2}{c|}{Base} & \multicolumn{2}{c|}{New} & \multirow{2}[4]{*}{$\mathrm{H}_\mathrm{b}$} \\
\cmidrule{2-5}          & Acc.   & Rob.   & Acc.   & Rob.   &  \\
    \midrule
    TeCoA & 54.38  & 40.07  & 55.57  & 41.64  & 46.86  \\
    \midrule
    APT   & 71.45  & 55.30  & 43.97  & 31.68  & 46.30  \\
    \rowcolor{Gray} ADAPT & \textbf{72.81} & \textbf{57.91} & 51.03  & 38.03  & 52.02  \\
    \midrule
    FAP   & 72.36  & 57.27  & 56.62  & 43.34  & 55.54  \\
    \rowcolor{Gray} ADAPT$_{\text M}$ & 72.18  & 56.56  & \textbf{61.21} & \textbf{45.89} & \textbf{57.42} \\
    \bottomrule
    \end{tabular}}}
    \end{minipage}
}
\hfill
\subfloat[ImageNet]{
    \begin{minipage}[t]{0.32\textwidth}
    \scalebox{0.72}{
    \setlength{\tabcolsep}{1.5mm}{
    \begin{tabular}{l|cc|cc|c}
    \toprule
    \multicolumn{1}{c|}{\multirow{2}[4]{*}{Methods}} & \multicolumn{2}{c|}{Base} & \multicolumn{2}{c|}{New} & \multirow{2}[4]{*}{$\mathrm{H}_\mathrm{b}$} \\
\cmidrule{2-5}          & Acc.   & Rob.   & Acc.   & Rob.   &  \\
    \midrule
    TeCoA & 58.91  & 40.75  & 59.71  & 44.16  & 49.44  \\
    \midrule
    APT   & 62.14  & 43.42  & 55.46  & 40.07  & 48.71  \\
    \rowcolor{Gray} ADAPT & 63.28  & 44.34  & 59.00  & 43.63  & \textbf{51.13} \\
    \midrule
    FAP   & 62.44  & 43.61  & 57.18  & 41.90  & 49.81  \\
    \rowcolor{Gray} ADAPT$_{\text M}$ & 63.24  & 43.20  & 59.08  & 42.66  & 50.42  \\
    \bottomrule
    \end{tabular}}}
    \end{minipage}
}
\hfill
\subfloat[Caltech101]{
    \begin{minipage}[t]{0.32\textwidth}
    \scalebox{0.72}{
    \setlength{\tabcolsep}{1.5mm}{
    \begin{tabular}{l|cc|cc|c}
    \toprule
    \multicolumn{1}{c|}{\multirow{2}[4]{*}{Methods}} & \multicolumn{2}{c|}{Base} & \multicolumn{2}{c|}{New} & \multirow{2}[4]{*}{$\mathrm{H}_\mathrm{b}$} \\
\cmidrule{2-5}          & Acc.   & Rob.   & Acc.   & Rob.   &  \\
    \midrule
    TeCoA & 89.35  & 79.41  & 89.85  & 82.86  & 85.14  \\
    \midrule
    APT   & 96.32  & 90.06  & 85.92  & 77.40  & 86.87  \\
    \rowcolor{Gray} ADAPT & 96.84  & 91.03  & 87.99  & 80.13  & \textbf{88.58} \\
    \midrule
    FAP   & 95.35  & 89.15  & 87.55  & 80.35  & 87.77  \\
    \rowcolor{Gray} ADAPT$_{\text M}$ & 95.67  & 88.77  & 88.32  & 81.00  & 88.13  \\
    \bottomrule
    \end{tabular}}}
    \end{minipage}
}
\\
\subfloat[OxfordPets]{
    \begin{minipage}[t]{0.32\textwidth}
    \scalebox{0.72}{
    \setlength{\tabcolsep}{1.5mm}{
    \begin{tabular}{l|cc|cc|c}
    \toprule
    \multicolumn{1}{c|}{\multirow{2}[4]{*}{Methods}} & \multicolumn{2}{c|}{Base} & \multicolumn{2}{c|}{New} & \multirow{2}[4]{*}{$\mathrm{H}_\mathrm{b}$} \\
\cmidrule{2-5}          & Acc.   & Rob.   & Acc.   & Rob.   &  \\
    \midrule
    TeCoA & 84.90  & 70.18  & 91.50  & 79.42  & 80.73  \\
    \midrule
    APT   & 87.40  & 71.08  & 82.83  & 69.85  & 77.07  \\
    \rowcolor{Gray} ADAPT & 90.54  & 77.46  & 91.83  & 81.04  & \textbf{84.77} \\
    \midrule
    FAP   & 89.79  & 76.56  & 91.05  & 79.92  & 83.87  \\
    \rowcolor{Gray} ADAPT$_{\text M}$ & 89.37  & 76.93  & 89.49  & 78.19  & 83.07  \\
    \bottomrule
    \end{tabular}}}
    \end{minipage}
}
\hfill
\subfloat[StanfordCars]{
    \begin{minipage}[t]{0.32\textwidth}
    \scalebox{0.72}{
    \setlength{\tabcolsep}{1.5mm}{
    \begin{tabular}{l|cc|cc|c}
    \toprule
    \multicolumn{1}{c|}{\multirow{2}[4]{*}{Methods}} & \multicolumn{2}{c|}{Base} & \multicolumn{2}{c|}{New} & \multirow{2}[4]{*}{$\mathrm{H}_\mathrm{b}$} \\
\cmidrule{2-5}          & Acc.   & Rob.   & Acc.   & Rob.   &  \\
    \midrule
    TeCoA & 27.94  & 14.79  & 36.27  & 19.68  & 22.00  \\
    \midrule
    APT   & 63.92  & 40.50  & 28.20  & 15.62  & 28.61  \\
    \rowcolor{Gray} ADAPT & 63.37  & 41.23  & 31.54  & 17.26  & 30.85  \\
    \midrule
    FAP   & 61.59  & 36.21  & 50.26  & 29.98  & 41.19  \\
    \rowcolor{Gray} ADAPT$_{\text M}$ & 59.85  & 34.31  & 56.23  & 33.20  & \textbf{42.66} \\
    \bottomrule
    \end{tabular}}}
    \end{minipage}
}
\hfill
\subfloat[Flowers102]{
    \begin{minipage}[t]{0.32\textwidth}
    \scalebox{0.72}{
    \setlength{\tabcolsep}{1.5mm}{
    \begin{tabular}{l|cc|cc|c}
    \toprule
    \multicolumn{1}{c|}{\multirow{2}[4]{*}{Methods}} & \multicolumn{2}{c|}{Base} & \multicolumn{2}{c|}{New} & \multirow{2}[4]{*}{$\mathrm{H}_\mathrm{b}$} \\
\cmidrule{2-5}          & Acc.   & Rob.   & Acc.   & Rob.   &  \\
    \midrule
    TeCoA & 56.79  & 42.07  & 58.30  & 40.64  & 48.11  \\
    \midrule
    APT   & 90.41  & 80.91  & 31.28  & 18.87  & 36.91  \\
    \rowcolor{Gray} ADAPT & 91.93  & 82.62  & 42.77  & 28.51  & 49.12  \\
    \midrule
    FAP   & 89.93  & 79.49  & 51.28  & 37.66  & 57.34  \\
    \rowcolor{Gray} ADAPT$_{\text M}$ & 89.55  & 77.21  & 56.88  & 43.05  & \textbf{61.61} \\
    \bottomrule
    \end{tabular}}}
    \end{minipage}
}
\hfill
\\
\subfloat[Food101]{
    \begin{minipage}[t]{0.32\textwidth}
    \scalebox{0.72}{
    \setlength{\tabcolsep}{1.5mm}{
    \begin{tabular}{l|cc|cc|c}
    \toprule
    \multicolumn{1}{c|}{\multirow{2}[4]{*}{Methods}} & \multicolumn{2}{c|}{Base} & \multicolumn{2}{c|}{New} & \multirow{2}[4]{*}{$\mathrm{H}_\mathrm{b}$} \\
\cmidrule{2-5}          & Acc.   & Rob.   & Acc.   & Rob.   &  \\
    \midrule
    TeCoA & 55.18  & 33.73  & 55.94  & 36.08  & 42.84  \\
    \midrule
    APT   & 60.41  & 37.76  & 45.21  & 26.10  & 38.66  \\
    \rowcolor{Gray} ADAPT & 63.63  & 42.20  & 50.42  & 30.60  & 43.51  \\
    \midrule
    FAP   & 69.76  & 47.42  & 66.25  & 43.90  & \textbf{54.57} \\
    \rowcolor{Gray} ADAPT$_{\text M}$ & 70.19  & 46.43  & 66.59  & 43.61  & 54.25  \\
    \bottomrule
    \end{tabular}}}
    \end{minipage}
}
\hfill
\subfloat[FGVCAircraft]{
    \begin{minipage}[t]{0.32\textwidth}
    \scalebox{0.72}{
    \setlength{\tabcolsep}{1.5mm}{
    \begin{tabular}{l|cc|cc|c}
    \toprule
    \multicolumn{1}{c|}{\multirow{2}[4]{*}{Methods}} & \multicolumn{2}{c|}{Base} & \multicolumn{2}{c|}{New} & \multirow{2}[4]{*}{$\mathrm{H}_\mathrm{b}$} \\
\cmidrule{2-5}          & Acc.   & Rob.   & Acc.   & Rob.   &  \\
    \midrule
    TeCoA & 15.31  & 8.28  & 18.30  & 9.48  & 11.55  \\
        \midrule
    APT   & 28.75  & 17.23  & 13.14  & 7.68  & 13.37  \\
    \rowcolor{Gray} ADAPT & 28.51  & 18.55  & 18.24  & 10.32  & 16.62  \\
        \midrule
    FAP   & 24.73  & 15.67  & 10.92  & 6.12  & 11.14  \\
    \rowcolor{Gray} ADAPT$_{\text M}$ & 25.75  & 17.71  & 23.58  & 13.62  & \textbf{18.95} \\
    \bottomrule
    \end{tabular}}}
    \end{minipage}
}
\hfill
\subfloat[SUN397]{
    \begin{minipage}[t]{0.32\textwidth}
    \scalebox{0.72}{
    \setlength{\tabcolsep}{1.5mm}{
    \begin{tabular}{l|cc|cc|c}
    \toprule
    \multicolumn{1}{c|}{\multirow{2}[4]{*}{Methods}} & \multicolumn{2}{c|}{Base} & \multicolumn{2}{c|}{New} & \multirow{2}[4]{*}{$\mathrm{H}_\mathrm{b}$} \\
\cmidrule{2-5}          & Acc.   & Rob.   & Acc.   & Rob.   &  \\
    \midrule
    TeCoA & 56.81  & 38.40  & 61.19  & 43.61  & 48.24  \\
        \midrule
    APT   & 70.20  & 49.55  & 48.90  & 32.63  & 46.77  \\
    \rowcolor{Gray} ADAPT & 70.61  & 51.31  & 56.35  & 39.11  & 51.97  \\
        \midrule
    FAP   & 70.35  & 50.82  & 62.85  & 44.47  & 55.33  \\
    \rowcolor{Gray} ADAPT$_{\text M}$ & 70.52  & 49.59  & 64.43  & 46.06  & \textbf{55.88} \\
    \bottomrule
    \end{tabular}}}
    \end{minipage}
}
\hfill
\\
\subfloat[DTD]{
    \begin{minipage}[t]{0.32\textwidth}
    \scalebox{0.72}{
    \setlength{\tabcolsep}{1.5mm}{
    \begin{tabular}{l|cc|cc|c}
    \toprule
    \multicolumn{1}{c|}{\multirow{2}[4]{*}{Methods}} & \multicolumn{2}{c|}{Base} & \multicolumn{2}{c|}{New} & \multirow{2}[4]{*}{$\mathrm{H}_\mathrm{b}$} \\
\cmidrule{2-5}          & Acc.   & Rob.   & Acc.   & Rob.   &  \\
    \midrule
    TeCoA & 44.10  & 32.29  & 42.87  & 33.21  & 37.35  \\
        \midrule
    APT   & 64.35  & 48.03  & 33.70  & 24.88  & 37.66  \\
    \rowcolor{Gray} ADAPT & 66.67  & 51.62  & 40.58  & 31.04  & 43.84  \\
        \midrule
    FAP   & 68.75  & 55.67  & 41.55  & 32.85  & 45.97  \\
    \rowcolor{Gray} ADAPT$_{\text M}$ & 69.79  & 55.67  & 49.40  & 36.96  & \textbf{50.26} \\
    \bottomrule
    \end{tabular}}}
    \end{minipage}
}
\hfill
\subfloat[EuroSAT]{
    \begin{minipage}[t]{0.32\textwidth}
    \scalebox{0.72}{
    \setlength{\tabcolsep}{1.5mm}{
    \begin{tabular}{l|cc|cc|c}
    \toprule
    \multicolumn{1}{c|}{\multirow{2}[4]{*}{Methods}} & \multicolumn{2}{c|}{Base} & \multicolumn{2}{c|}{New} & \multirow{2}[4]{*}{$\mathrm{H}_\mathrm{b}$} \\
\cmidrule{2-5}          & Acc.   & Rob.   & Acc.   & Rob.   &  \\
    \midrule
    TeCoA & 52.74  & 40.00  & 38.92  & 28.03  & 37.98  \\
    \midrule
    APT   & 87.02  & 72.26  & 24.23  & 10.97  & 25.36  \\
    \rowcolor{Gray} ADAPT & 89.07  & 75.19  & 38.79  & 27.77  & 46.34  \\
    \midrule
    FAP   & 87.69  & 76.43  & 49.87  & 41.38  & 58.22  \\
    \rowcolor{Gray} ADAPT$_{\text M}$ & 84.93  & 73.33  & 62.41  & 47.31  & \textbf{63.93} \\
    \bottomrule
    \end{tabular}}}
    \end{minipage}
}
\hfill
\subfloat[UCF101]{
    \begin{minipage}[t]{0.32\textwidth}
    \scalebox{0.72}{
    \setlength{\tabcolsep}{1.5mm}{
    \begin{tabular}{l|cc|cc|c}
    \toprule
    \multicolumn{1}{c|}{\multirow{2}[4]{*}{Methods}} & \multicolumn{2}{c|}{Base} & \multicolumn{2}{c|}{New} & \multirow{2}[4]{*}{$\mathrm{H}_\mathrm{b}$} \\
\cmidrule{2-5}          & Acc.   & Rob.   & Acc.   & Rob.   &  \\
    \midrule
    TeCoA & 56.10  & 40.85  & 58.41  & 40.83  & 47.67  \\
    \midrule
    APT   & 75.03  & 57.50  & 34.83  & 24.39  & 39.83  \\
    \rowcolor{Gray} ADAPT & 76.42  & 61.43  & 43.81  & 28.88  & 46.07  \\
    \midrule
    FAP   & 75.54  & 58.89  & 54.03  & 38.18  & 53.39  \\
    \rowcolor{Gray} ADAPT$_{\text M}$ & 75.13  & 59.05  & 56.90  & 39.16  & \textbf{54.53} \\
    \bottomrule
    \end{tabular}}}
    \end{minipage}
}
\label{appendix:table_base2new}
\end{table*}

\subsection{Full Results of Adversarial Few-shot Classification}
\label{appendix:fewshot}

In this section, we provide the full results of adversarial few-shot classification performance on 11 datasets. We evaluate the models using 1, 4, and 16 shots per class. Table~\ref{tab:fewshot} shows that the proposed methods demonstrate good data efficiency compared to the baselines. 
In the average performance across all 11 datasets, ADAPT$_{\text M}$ consistently achieves the highest `H' score across all shot settings (i.e., 19.96\% for 1-shot, 24.78\% for 4-shot, and 30.11\% for 16-shot).
Those results demonstrate that the proposed dual-prompt mechanism with orthogonal loss effectively mitigates the reliance on shortcuts and is highly effective in learning robust representations even when labeled data is extremely scarce.

\begin{table*}[htbp]
  \centering
  \caption{The performance on the 11 datasets for different shots under the adversarial few-shot classification setting.}
  \resizebox{0.8\linewidth}{!}{
    \setlength{\tabcolsep}{2mm}{
    \begin{tabular}{cc|c|cccc|cccc|cccc}
    \toprule
    \multirow{2}[4]{*}{Dataset} & \multirow{2}[4]{*}{$\epsilon=4/255$} & Zero-shot & \multicolumn{4}{c|}{1 shot}   & \multicolumn{4}{c|}{4 shot}   & \multicolumn{4}{c}{16 shot} \\
\cmidrule{3-15}          &       & TeCoA & APT   & FAP   & \cellcolor{Gray}ADAPT & \cellcolor{Gray}ADAPT$_{\text M}$ & APT   & FAP   & \cellcolor{Gray}ADAPT & \cellcolor{Gray}ADAPT$_{\text M}$ & APT   & FAP   & \cellcolor{Gray}ADAPT & \multicolumn{1}{c}{\cellcolor{Gray}ADAPT$_{\text M}$} \\
    \midrule
    \multirow{3}[2]{*}{\textbf{Average}} & Acc.  & 33.67  & 33.20  & 37.78  & \cellcolor{Gray}\textbf{38.10} & \cellcolor{Gray}38.06  & 43.32  & 44.42  & \cellcolor{Gray}44.89 & \cellcolor{Gray}\textbf{45.09}  & 51.08  & 51.11  & \cellcolor{Gray}\textbf{52.68} & \cellcolor{Gray}49.89  \\
          & Rob.  & 10.79  & 11.79  & 13.19  & \cellcolor{Gray}12.69  & \cellcolor{Gray}\textbf{13.53} & 14.66  & 15.31  & \cellcolor{Gray}15.80 & \cellcolor{Gray}\textbf{17.09}  & 20.26  & 19.65  & \cellcolor{Gray}20.51  & \cellcolor{Gray}\textbf{21.56} \\
          & H     & 16.34  & 17.40  & 19.55  & \cellcolor{Gray}19.04  & \cellcolor{Gray}\textbf{19.96} & 21.91  & 22.77  & \cellcolor{Gray}23.37 & \cellcolor{Gray}\textbf{24.78}  & 29.02  & 28.39  & \cellcolor{Gray}29.53  & \cellcolor{Gray}\textbf{30.11} \\
    \midrule
    \multirow{3}[2]{*}{ImageNet} & Acc.  & 40.11  & 37.88  & 36.95  & \cellcolor{Gray}\textbf{39.85} & \cellcolor{Gray}36.62  & 39.13  & 38.94  & \cellcolor{Gray}\textbf{40.84} & \cellcolor{Gray}37.60  & 41.06  & 40.45  & \cellcolor{Gray}\textbf{42.27} & \cellcolor{Gray}38.13  \\
          & Rob.  & 10.14  & 10.79  & \textbf{11.41} & \cellcolor{Gray}10.96  & \cellcolor{Gray}11.05  & 11.49  & \textbf{11.86} & \cellcolor{Gray}11.15  & \cellcolor{Gray}11.83  & 12.02  & 12.03  & \cellcolor{Gray}11.64  & \cellcolor{Gray}\textbf{12.19} \\
          & H     & 16.19  & 16.80  & \textbf{17.44} & \cellcolor{Gray}17.19  & \cellcolor{Gray}16.98  & 17.76  & \textbf{18.18} & \cellcolor{Gray}17.52  & \cellcolor{Gray}18.00  & 18.60  & \textbf{18.54} & \cellcolor{Gray}18.25  & \cellcolor{Gray}18.47  \\
    \midrule
    \multirow{3}[2]{*}{Caltech101} & Acc.  & 78.78  & 77.89  & 76.67  & \cellcolor{Gray}\textbf{81.05} & \cellcolor{Gray}79.19  & 81.62  & 77.61  & \cellcolor{Gray}\textbf{84.71} & \cellcolor{Gray}82.76  & 86.29  & 80.65  & \cellcolor{Gray}\textbf{88.80} & \cellcolor{Gray}84.62  \\
          & Rob.  & 43.61  & 45.84  & \textbf{48.36} & \cellcolor{Gray}45.72  & \cellcolor{Gray}48.32  & 47.59  & 51.81  & \cellcolor{Gray}\textbf{52.78} & \cellcolor{Gray}52.01  & 56.75  & 55.38  & \cellcolor{Gray}\textbf{57.36} & \cellcolor{Gray}57.24  \\
          & H     & 56.14  & 57.71  & 59.31  & \cellcolor{Gray}58.46  & \cellcolor{Gray}\textbf{60.02} & 60.12  & 62.14  & \cellcolor{Gray}\textbf{65.04} & \cellcolor{Gray}63.88  & 68.47  & 65.67  & \cellcolor{Gray}\textbf{69.70} & \cellcolor{Gray}68.29  \\
    \midrule
    \multirow{3}[2]{*}{OxfordPets} & Acc.  & 66.26  & 59.58  & 62.14  & \cellcolor{Gray}62.12  & \cellcolor{Gray}\textbf{63.40} & 65.99  & 64.57  & \cellcolor{Gray}\textbf{69.96} & \cellcolor{Gray}67.95  & 67.29  & 67.18  & \cellcolor{Gray}\textbf{72.04} & \cellcolor{Gray}68.06  \\
          & Rob.  & 15.56  & 15.54  & \textbf{19.81} & \cellcolor{Gray}15.70  & \cellcolor{Gray}17.85  & 17.80  & \textbf{22.24} & \cellcolor{Gray}18.15  & \cellcolor{Gray}21.59  & 19.98  & \textbf{24.88} & \cellcolor{Gray}20.63  & \cellcolor{Gray}24.12  \\
          & H     & 25.20  & 24.65  & \textbf{30.04} & \cellcolor{Gray}25.07  & \cellcolor{Gray}27.86  & 28.04  & \textbf{33.08} & \cellcolor{Gray}28.82  & \cellcolor{Gray}32.77  & 30.81  & \textbf{36.31} & \cellcolor{Gray}32.07  & \cellcolor{Gray}35.62  \\
    \midrule
    \multicolumn{1}{c}{\multirow{3}[2]{*}{Stanford\newline{}Cars}} & Acc.  & 10.32  & 19.49  & \textbf{27.41} & \cellcolor{Gray}21.23  & \cellcolor{Gray}26.90  & 23.93  & 33.96  & \cellcolor{Gray}28.22  & \cellcolor{Gray}\textbf{34.05} & 31.60  & \textbf{41.54} & \cellcolor{Gray}35.78  & \cellcolor{Gray}38.68  \\
          & Rob.  & 0.99  & 2.90  & 3.27  & \cellcolor{Gray}3.48  & \cellcolor{Gray}\textbf{3.52} & 4.44  & 4.76  & \cellcolor{Gray}4.66  & \cellcolor{Gray}\textbf{5.32} & \textbf{7.70} & 6.62  & \cellcolor{Gray}7.35  & \cellcolor{Gray}7.18  \\
          & H     & 1.81  & 5.05  & 5.84  & \cellcolor{Gray}5.98  & \cellcolor{Gray}\textbf{6.23} & 7.49  & 8.35  & \cellcolor{Gray}8.00  & \cellcolor{Gray}\textbf{9.20} & \textbf{12.38} & 11.42  & \cellcolor{Gray}12.19  & \cellcolor{Gray}12.11  \\
    \midrule
    \multirow{3}[2]{*}{Flowers102} & Acc.  & 30.13  & 35.93  & \textbf{42.55} & \cellcolor{Gray}36.50  & \cellcolor{Gray}37.23  & 60.25  & 55.58  & \cellcolor{Gray}\textbf{60.74} & \cellcolor{Gray}57.46  & \textbf{76.41} & 66.26  & \cellcolor{Gray}75.03  & \cellcolor{Gray}71.38  \\
          & Rob.  & 8.93  & 11.49  & \textbf{14.21} & \cellcolor{Gray}10.19  & \cellcolor{Gray}13.93  & 22.45  & 20.46  & \cellcolor{Gray}\textbf{22.98} & \cellcolor{Gray}21.40  & \textbf{37.52} & 30.25  & \cellcolor{Gray}37.03  & \cellcolor{Gray}33.41  \\
          & H     & 13.78  & 17.41  & \textbf{21.30} & \cellcolor{Gray}15.93  & \cellcolor{Gray}20.27  & 32.71  & 29.91  & \cellcolor{Gray}\textbf{33.34} & \cellcolor{Gray}31.19  & \textbf{50.33} & 41.54  & \cellcolor{Gray}49.59  & \cellcolor{Gray}45.52  \\
    \midrule
    \multirow{3}[2]{*}{Food101} & Acc.  & 23.52  & 20.73  & \textbf{33.46} & \cellcolor{Gray}26.65  & \cellcolor{Gray}30.38  & 21.97  & \textbf{38.84} & \cellcolor{Gray}29.36  & \cellcolor{Gray}31.66  & 30.39  & \textbf{44.35} & \cellcolor{Gray}35.14  & \cellcolor{Gray}34.86  \\
          & Rob.  & 3.27  & 3.42  & 5.46  & \cellcolor{Gray}4.97  & \cellcolor{Gray}\textbf{6.62} & 4.01  & 6.39  & \cellcolor{Gray}5.64  & \cellcolor{Gray}\textbf{7.73} & 7.90  & 8.31  & \cellcolor{Gray}7.95  & \cellcolor{Gray}\textbf{9.94} \\
          & H     & 5.74  & 5.87  & 9.39  & \cellcolor{Gray}8.38  & \cellcolor{Gray}\textbf{10.87} & 6.78  & 10.97  & \cellcolor{Gray}9.46  & \cellcolor{Gray}\textbf{12.43} & 12.54  & 14.00  & \cellcolor{Gray}12.97  & \cellcolor{Gray}\textbf{15.47} \\
    \midrule
    \multicolumn{1}{c}{\multirow{3}[2]{*}{FGVC\newline{}Aircraft}} & Acc.  & 7.17  & 3.90  & 11.94  & \cellcolor{Gray}10.11  & \cellcolor{Gray}\textbf{13.08} & 13.65  & \textbf{15.27} & \cellcolor{Gray}10.29  & \cellcolor{Gray}14.37  & \textbf{20.31} & 18.39  & \cellcolor{Gray}19.29  & \cellcolor{Gray}19.08  \\
          & Rob.  & 0.36  & 2.19  & 2.52  & \cellcolor{Gray}2.16  & \cellcolor{Gray}\textbf{3.15} & 3.33  & 2.73  & \cellcolor{Gray}\textbf{3.60} & \cellcolor{Gray}3.54  & 6.15  & 5.34  & \cellcolor{Gray}5.79  & \cellcolor{Gray}\textbf{6.75} \\
          & H     & 0.69  & 2.80  & 4.16  & \cellcolor{Gray}3.56  & \cellcolor{Gray}\textbf{5.08} & 5.35  & 4.63  & \cellcolor{Gray}5.33  & \cellcolor{Gray}\textbf{5.68} & 9.44  & 8.28  & \cellcolor{Gray}8.91  & \cellcolor{Gray}\textbf{9.97} \\
    \midrule
    \multirow{3}[2]{*}{SUN397} & Acc.  & 33.24  & 32.32  & \textbf{36.23} & \cellcolor{Gray}35.86  & \cellcolor{Gray}33.63  & 39.06  & 39.68  & \cellcolor{Gray}\textbf{40.98} & \cellcolor{Gray}38.68  & 45.21  & 43.35  & \cellcolor{Gray}\textbf{46.10} & \cellcolor{Gray}42.29  \\
          & Rob.  & 6.20  & 5.76  & 8.16  & \cellcolor{Gray}7.25  & \cellcolor{Gray}\textbf{8.32} & 7.92  & 9.69  & \cellcolor{Gray}8.70  & \cellcolor{Gray}\textbf{9.92} & 11.27  & \textbf{11.64} & \cellcolor{Gray}11.04  & \cellcolor{Gray}11.62  \\
          & H     & 10.45  & 9.78  & 13.32  & \cellcolor{Gray}12.06  & \cellcolor{Gray}\textbf{13.34} & 13.17  & 15.58  & \cellcolor{Gray}14.35  & \cellcolor{Gray}\textbf{15.79} & 18.04  & \textbf{18.35} & \cellcolor{Gray}17.81  & \cellcolor{Gray}18.23  \\
    \midrule
    \multirow{3}[2]{*}{DTD} & Acc.  & 24.29  & 23.29  & 24.65  & \cellcolor{Gray}28.55  & \cellcolor{Gray}\textbf{29.02} & 36.17  & \textbf{37.41} & \cellcolor{Gray}37.06  & \cellcolor{Gray}37.17  & \textbf{45.86} & 44.21  & \cellcolor{Gray}44.39  & \cellcolor{Gray}43.68  \\
          & Rob.  & 11.35  & 9.46  & 9.04  & \cellcolor{Gray}\textbf{12.35} & \cellcolor{Gray}12.23  & 14.13  & 15.60  & \cellcolor{Gray}14.54  & \cellcolor{Gray}\textbf{17.55} & 21.51  & 21.63  & \cellcolor{Gray}\textbf{22.64} & \cellcolor{Gray}22.52  \\
          & H     & 15.47  & 13.45  & 13.23  & \cellcolor{Gray}\textbf{17.24} & \cellcolor{Gray}17.21  & 20.32  & 22.02  & \cellcolor{Gray}20.89  & \cellcolor{Gray}\textbf{23.84} & 29.28  & 29.05  & \cellcolor{Gray}\textbf{29.99} & \cellcolor{Gray}29.72  \\
    \midrule
    \multirow{3}[2]{*}{EuroSAT} & Acc.  & 19.56  & 21.09  & 28.02  & \cellcolor{Gray}\textbf{37.28} & \cellcolor{Gray}31.36  & \textbf{49.98} & 40.58  & \cellcolor{Gray}41.90  & \cellcolor{Gray}48.30  & 64.33  & \textbf{64.51} & \cellcolor{Gray}64.04  & \cellcolor{Gray}56.23  \\
          & Rob.  & 11.22  & 14.53  & 13.94  & \cellcolor{Gray}\textbf{15.74} & \cellcolor{Gray}13.06  & 16.44  & 8.57  & \cellcolor{Gray}19.40  & \cellcolor{Gray}\textbf{21.51} & 25.54  & 21.28  & \cellcolor{Gray}25.93  & \cellcolor{Gray}\textbf{32.27} \\
          & H     & 14.26  & 17.21  & 18.62  & \cellcolor{Gray}\textbf{22.13} & \cellcolor{Gray}18.44  & 24.74  & 14.15  & \cellcolor{Gray}26.52  & \cellcolor{Gray}\textbf{29.76} & 36.56  & 32.00  & \cellcolor{Gray}36.91  & \cellcolor{Gray}\textbf{41.01} \\
    \midrule
    \multirow{3}[2]{*}{UCF101} & Acc.  & 36.98  & 33.10  & 35.58  & \cellcolor{Gray}\textbf{39.86} & \cellcolor{Gray}37.83  & 44.75  & 46.21  & \cellcolor{Gray}\textbf{49.78} & \cellcolor{Gray}46.00  & 53.16  & 51.28  & \cellcolor{Gray}\textbf{56.62} & \cellcolor{Gray}51.78  \\
          & Rob.  & 7.01  & 7.72  & 8.86  & \cellcolor{Gray}\textbf{11.05} & \cellcolor{Gray}10.73  & 11.66  & 14.27  & \cellcolor{Gray}12.16  & \cellcolor{Gray}\textbf{15.57} & 16.55  & 18.82  & \cellcolor{Gray}18.29  & \cellcolor{Gray}\textbf{19.96} \\
          & H     & 11.79  & 12.52  & 14.19  & \cellcolor{Gray}\textbf{17.30} & \cellcolor{Gray}16.72  & 18.50  & 21.81  & \cellcolor{Gray}19.55  & \cellcolor{Gray}\textbf{23.27} & 25.24  & 27.53  & \cellcolor{Gray}27.65  & \cellcolor{Gray}\textbf{28.81} \\
    \bottomrule
    \end{tabular}}}
  \label{tab:fewshot}%
\end{table*}

\subsection{Results of Adversarial Domain Generalization}
\label{appendix:dg}

\noindent \textbf{Adversarial domain generalization.}
In this scenario, we assess the capability of prompts for OOD data. Specifically, models are tuned using 16-shot samples from each of the 1,000 classes on ImageNet (source), and then evaluated on four different target domains (i.e., ImageNet-V2, ImageNet-Sketch, ImageNet-A, and ImageNet-R).
Table \ref{tab_dg} shows that ADAPT achieves the highest accuracy of 41.34\% and competitive robustness compared to baselines on the source domain.
More importantly, ADAPT achieves the best average performance of 23.92\%, 7.84\%, and 11.81\% in accuracy, robustness, and `H', respectively. 
Specifically, ADAPT surpasses APT by substantial margins on challenging domains like ImageNet-Sketch and ImageNet-R, achieving `H' score gains of 0.40\% and 0.85\%, respectively.
Those results show that the robustness learned by ADAPT is not merely due to overfitting the source distribution shortcuts, but stems from disentangled features that can be effectively transferred to OOD data.

\begin{table*}[!t]
  \centering
  \caption{Performance under adversarial domain generalization setting. $\epsilon=4/255$.}
  \resizebox{\linewidth}{!}{
  \setlength{\tabcolsep}{2mm}{
    \begin{tabular}{l|ccc|cccccccccccc|ccc}
    \toprule
    \multirow{3}[6]{*}{Method} & \multicolumn{3}{c|}{Source} & \multicolumn{15}{c}{Target} \\
\cmidrule{2-19}          & \multicolumn{3}{c|}{ImageNet} & \multicolumn{3}{c}{ImageNet-V2} & \multicolumn{3}{c}{ImageNet-Sketch} & \multicolumn{3}{c}{ImageNet-A} & \multicolumn{3}{c|}{ImageNet-R} & \multicolumn{3}{c}{Average} \\
\cmidrule{2-19}          & Acc.  & Rob.  & H     & Acc.  & Rob.  & H     & Acc.  & Rob.  & H     & Acc.  & Rob.  & H     & Acc.  & Rob.  & \multicolumn{1}{c}{H} & Acc.  & Rob.  & H \\
    \midrule
    TeCoA & 40.11  & 10.14  & 16.19  & 33.11  & 7.49  & 12.22  & 17.59  & 7.24  & 10.26  & 4.04  & 0.28  & 0.52  & 37.52  & 12.51  & 18.76  & 23.07  & 6.88  & 10.60  \\
    \midrule
    APT   & 41.06  & 12.02  & 18.60  & 33.67  & 9.09  & 14.32  & 18.22  & 7.87  & 10.99  & 4.19  & 0.36  & 0.66  & 37.04  & 13.29  & 19.56  & 23.28  & 7.65  & 11.52  \\
    FAP   & 40.32  & 12.06  & 18.57  & 32.81  & \textbf{9.17} & \textbf{14.33} & 16.42  & 7.11  & 9.92  & 3.87  & \textbf{0.43} & \textbf{0.77} & 36.04  & 13.55  & 19.70  & 22.29  & 7.57  & 11.30  \\
    \rowcolor{Gray} ADAPT & \textbf{41.34} & 12.02  & \textbf{18.62} & \textbf{34.24} & 8.92  & 14.15  & \textbf{18.56} & \textbf{8.22} & \textbf{11.39} & \textbf{4.21} & 0.36  & 0.66  & \textbf{38.65} & \textbf{13.87} & \textbf{20.41} & \textbf{23.92} & \textbf{7.84} & \textbf{11.81} \\
    \rowcolor{Gray} ADAPT$_{\text M}$ & 38.20  & \textbf{12.22} & 18.52  & 31.35  & 8.90  & 13.86  & 16.84  & 7.55  & 10.43  & 3.49  & \textbf{0.43} & \textbf{0.77} & 35.79  & 13.34  & 19.44  & 21.87  & 7.56  & 11.23  \\
    \bottomrule
    \end{tabular}}}
  \label{tab_dg}%
\end{table*}%

\subsection{Results of Adversarial Cross-dataset Generalization}
\label{appendix:xd}
To further evaluate the generalization capability of the learned prompts across different distributions, we conduct experiments under the adversarial cross-dataset generalization setting. 
We train the models on ImageNet (Source) using 16 shots per class and evaluate them directly on 10 other datasets (Target).
As shown in Table~\ref{tab_xd}, our proposed methods consistently outperform the baselines in terms of overall performance. 
ADAPT$_{\text M}$ achieves the highest average harmonic mean of 17.43\%. 
Those results demonstrate that the disentangled features learned by our methods are not only robust on the source domain but also highly transferable to unseen domains.

\begin{table*}[ht]
  \centering
  \caption{Performance under adversarial cross-dataset generalization setting. $\epsilon=4/255$.}
    \resizebox{\linewidth}{!}{
    \setlength{\tabcolsep}{1mm}{
    \begin{tabular}{cl|c|ccccccccccc}
    \toprule
    \multirow{2}[4]{*}{16 shot} & \multirow{2}[4]{*}{Method} & Source & \multicolumn{11}{c}{Target} \\
\cmidrule{3-14}          &       & ImageNet & Caltech101 & OxfordPets & StanfordCars & Flowers102 & Food101 & FGVCAircraft & SUN397 & DTD   & EuroSAT & UCF101 & \textbf{Avg.} \\
    \midrule
    \multirow{5}[2]{*}{Acc.} & TeCoA & 40.11  & 78.78  & \textbf{66.26} & 10.32  & 30.13  & \textbf{23.52} & \textbf{7.17} & \textbf{33.24} & \textbf{24.29} & 19.56  & \textbf{36.98} & \textbf{33.03} \\
          & APT   & 41.06  & 77.40  & 64.19  & \textbf{11.21} & 27.65  & 23.77  & 4.17  & 31.85  & 23.52  & 16.98  & 33.73  & 31.45  \\
          & FAP   & 40.32  & 79.19  & 63.67  & 9.64  & \textbf{30.25} & 22.82  & 5.43  & 31.92  & 23.23  & 16.79  & 32.25  & 31.52  \\
          & \cellcolor{Gray}ADAPT & \cellcolor{Gray}\textbf{41.34} & \cellcolor{Gray}\textbf{79.80} & \cellcolor{Gray}65.85  & \cellcolor{Gray}10.55  & \cellcolor{Gray}29.60  & \cellcolor{Gray}23.28  & \cellcolor{Gray}5.91  & \cellcolor{Gray}32.57  & \cellcolor{Gray}24.05  & \cellcolor{Gray}17.40  & \cellcolor{Gray}34.97  & \cellcolor{Gray}32.40  \\
          & \cellcolor{Gray}ADAPT$_{\text M}$ & \cellcolor{Gray}38.20  & \cellcolor{Gray}78.86  & \cellcolor{Gray}60.92  & \cellcolor{Gray}9.23  & \cellcolor{Gray}29.92  & \cellcolor{Gray}20.34  & \cellcolor{Gray}6.30  & \cellcolor{Gray}28.90  & \cellcolor{Gray}22.28  & \cellcolor{Gray}\textbf{22.74} & \cellcolor{Gray}30.43  & \cellcolor{Gray}30.99  \\
    \midrule
    \multirow{5}[2]{*}{Rob.} & TeCoA & 10.14  & 43.61  & 15.56  & 0.99  & 8.93  & 3.27  & 0.36  & 6.20  & \textbf{11.35} & 11.22  & 7.01  & 10.85  \\
          & APT   & 12.02  & 45.40  & 20.44  & 1.53  & 9.05  & \textbf{3.90} & 0.54  & \textbf{7.22} & 10.82  & 11.19  & 6.93  & 11.70  \\
          & FAP   & 12.06  & 46.25  & \textbf{20.71} & \textbf{1.55} & 9.74  & 3.67  & 0.66  & 7.15  & 10.82  & 11.38  & 7.45  & 11.94  \\
          & \cellcolor{Gray}ADAPT & \cellcolor{Gray}12.02  & \cellcolor{Gray}45.27  & \cellcolor{Gray}19.90  & \cellcolor{Gray}1.44  & \cellcolor{Gray}9.14  & \cellcolor{Gray}3.63  & \cellcolor{Gray}\textbf{1.14} & \cellcolor{Gray}6.84  & \cellcolor{Gray}10.82  & \cellcolor{Gray}11.30  & \cellcolor{Gray}\textbf{7.69} & \cellcolor{Gray}11.72  \\
          & \cellcolor{Gray}ADAPT$_{\text M}$ & \cellcolor{Gray}\textbf{12.22} & \cellcolor{Gray}\textbf{46.73} & \cellcolor{Gray}19.51  & \cellcolor{Gray}1.38  & \cellcolor{Gray}\textbf{10.11} & \cellcolor{Gray}3.48  & \cellcolor{Gray}1.08  & \cellcolor{Gray}6.95  & \cellcolor{Gray}10.64  & \cellcolor{Gray}\textbf{14.58} & \cellcolor{Gray}6.82  & \cellcolor{Gray}\textbf{12.13} \\
    \midrule
    \multirow{5}[2]{*}{H} & TeCoA & 16.19  & 56.14  & 25.20  & 1.81  & 13.78  & 5.74  & 0.69  & 10.45  & \textbf{15.47} & 14.26  & 11.79  & 16.33  \\
          & APT   & 18.60  & 57.23  & 31.01  & \textbf{2.69} & 13.64  & \textbf{6.70} & 0.96  & \textbf{11.77} & 14.82  & 13.49  & 11.50  & 17.06  \\
          & FAP   & 18.57  & 58.40  & \textbf{31.25} & 2.67  & 14.74  & 6.32  & 1.18  & 11.68  & 14.76  & 13.57  & 12.10  & 17.32  \\
          & \cellcolor{Gray}ADAPT & \cellcolor{Gray}\textbf{18.62} & \cellcolor{Gray}57.77  & \cellcolor{Gray}30.56  & \cellcolor{Gray}2.53  & \cellcolor{Gray}13.97  & \cellcolor{Gray}6.28  & \cellcolor{Gray}\textbf{1.91} & \cellcolor{Gray}11.31  & \cellcolor{Gray}14.93  & \cellcolor{Gray}13.70  & \cellcolor{Gray}\textbf{12.61} & \cellcolor{Gray}17.21  \\
          & \cellcolor{Gray}ADAPT$_{\text M}$ & \cellcolor{Gray}18.52  & \cellcolor{Gray}\textbf{58.69} & \cellcolor{Gray}29.55  & \cellcolor{Gray}2.40  & \cellcolor{Gray}\textbf{15.11} & \cellcolor{Gray}5.94  & \cellcolor{Gray}1.84  & \cellcolor{Gray}11.21  & \cellcolor{Gray}14.40  & \cellcolor{Gray}\textbf{17.77} & \cellcolor{Gray}11.14  & \cellcolor{Gray}\textbf{17.43} \\
    \bottomrule
    \end{tabular}}}
  \label{tab_xd}%
\end{table*}%

\begin{table*}[htbp]
  \centering
  \caption{Computational cost and performance comparison on ImageNet under the adversarial base-to-new generalization setting with $\epsilon=4/255$}
    \resizebox{0.6\linewidth}{!}{
    \setlength{\tabcolsep}{2mm}{
    \begin{tabular}{lccc|cc|cc|c}
    \toprule
    \multicolumn{1}{l}{\multirow{2}[4]{*}{Method}} & {\multirow{2}[4]{*}{Train.Memory}} & \multirow{2}[4]{*}{Train.time (h)} & \multirow{2}[4]{*}{Infer.time (ms)} & \multicolumn{2}{c|}{Base} & \multicolumn{2}{c|}{New} & \multirow{2}[4]{*}{$\mathrm{H}_\mathrm{b}$} \\
\cmidrule{5-8}          &       &       &       & Acc   & Rob   & Acc   & Rob   &  \\
    \midrule
    APT   & 11454M & 5.83  & 1.70  & 44.54  & 13.78  & 38.56  & 12.47  & 19.89  \\
    \rowcolor{Gray} ADAPT & 27470M & 9.65  & 1.72  & \textbf{47.07} & 13.49  & \textbf{43.36} & 13.71  & 20.90  \\
    FAP   & 19808M & 11.22 & 1.97  & 43.29  & \textbf{13.89} & 37.43  & 13.05  & 20.16  \\
    \rowcolor{Gray} ADAPT$_{\text M}$ & 27530M & 13.81 & 1.69  & 41.67  & 13.88  & 39.47  & \textbf{14.49} & \textbf{21.01} \\
    \bottomrule
    \end{tabular}}}
  \label{tab:computational_cost}%
\end{table*}%

\subsection{Computational Cost Analysis}
\label{appendix:cost}

To evaluate the practicality and efficiency of our proposed method, we provide a comprehensive analysis of its computational overhead compared to APT and FAP. For a fair comparison, all experiments were conducted on Quadro RTX 8000 under the adversarial base-to-new generalization setting with $\epsilon=4/255$.

The results in \cref{tab:computational_cost} demonstrate that ADAPT achieves a highly favorable trade-off between computational cost and robustness.

\textbf{Training Memory.} The training memory required by ADAPT (27470M) and ADAPT$_{\text M}$ (27530M) is higher than that of APT (11454M) and FAP (19808M). This increase primarily stems from the maintenance of the pool of decoy prompt. However, it is worth noting that this memory consumption is proportional to the number of classes. Thus, the memory requirement decreases significantly on datasets with fewer categories than ImageNet.
We consider this training-time cost a justifiable trade-off for the good generalization performance achieved.

\textbf{Training Time.} In terms of training duration, ADAPT (9.65h) takes longer than the lightweight APT (5.83h) but is notably more efficient than FAP (11.22h). This indicates that while our dual-prompt mechanism introduces some computational overhead, it is still more time-efficient than existing multi-modal prompt tuning approaches.

\textbf{Inference Time.} Crucially, during the inference phase, the per-image processing time for ADAPT (1.72ms) and ADAPT$_{\text M}$ (1.69ms) is comparable to APT (1.70ms) and faster than FAP (1.97ms). This is because the decoy prompts are used solely for optimization constraints during training and are discarded during inference. Consequently, our method introduces zero computational overhead during deployment compared to standard adversarial prompt tuning, making it highly suitable for practical, real-time applications.

\textbf{Complexity and scalability.}
ADAPT introduces no additional inference cost over APT because the decoy
prompts are discarded after training. During each training iteration,
adversarial examples are generated once using the target prompt, and only
one decoy prompt is sampled and updated. Therefore, the decoy computation
scales with $\mathcal{O}(C)$ rather than $\mathcal{O}(KC)$, where $C$ is
the number of classes and $K$ is the size of the decoy pool. The main
overhead arises from retaining the text-embedding computation graph of the
active decoy prompt, and thus becomes more visible on datasets with many
classes.
On DTD with ViT-B/32, ADAPT
increases memory and training time from 2394M/0.209h to 2810M/0.276h.
With ViT-L/14, they increase from 9190M/0.95h to 9964M/1.10h. These
results indicate that the additional training cost remains manageable on a
larger backbone. This cost is most useful when robust transfer to unseen
classes is the main objective, for which ADAPT substantially improves over
APT while preserving strong base-class performance.

In summary, while ADAPT incurs a moderate increase in training resources to optimize the dual-prompt mechanism, it maintains high inference efficiency. This characteristic, combined with the substantial improvement in robustness, highlights the practicality of our method.

\subsection{Empirical Validation of the Semantic--Shortcut Decomposition}
\label{app:projection_analysis}

The semantic--shortcut decomposition in Sec. 3.4 is an analytical abstraction rather than an exact structural claim about CLIP features. To connect this abstraction with real CLIP representations, we conduct an
image-side projection analysis on DTD. With the CLIP encoders frozen and
the hand-crafted prompt used as the classifier, we learn a projection matrix
$P\in\mathbb{R}^{512\times512}$ on the base classes through adversarial
training. Each image feature $z$ is decomposed as
\begin{equation}
    z_U=z_P,\qquad z_S=z-z_U,
\end{equation}
where $z_U$ denotes the learned base-useful component and $z_S$ denotes its
residual. Both components are classified using the same hand-crafted prompt.

As shown in \cref{tab:projection_analysis}, $z_U$ performs better on
the base classes, whereas $z_S$ transfers better to the new classes. This
result suggests that real CLIP features can be separated into a base-useful
component and a more transferable residual component, providing empirical
support for the theoretical abstraction used in our analysis.

\begin{table}[ht]
    \centering
    \caption{Image-side projection analysis on DTD under adversarial base-to-new generalization.}
    \setlength{\tabcolsep}{4pt}
    \begin{tabular}{c|cc|cc}
        \toprule
        & $\mathrm{Acc.}_{\mathrm{Base}}$
        & $\mathrm{Rob.}_{\mathrm{Base}}$
        & $\mathrm{Acc.}_{\mathrm{New}}$
        & $\mathrm{Rob.}_{\mathrm{New}}$ \\
        \midrule
        $z_U$ & \textbf{37.27} & \textbf{15.74} & 22.45 & 10.65 \\
        $z_S$ & 19.93 & 9.78 & \textbf{30.56} & \textbf{14.98} \\
        \bottomrule
    \end{tabular}
    \label{tab:projection_analysis}
\end{table}

\subsection{Robust Generalization Overfitting under Different Training Settings}
\label{app:rgo_settings}

To examine whether robust generalization overfitting is caused by a
particular training configuration, we evaluate APT using learning rates
from 0.001 to 0.004 and training durations from 50 to 200 epochs.
\cref{fig:rgo_settings} shows that the robust loss gap between base
and new classes emerges across these settings, although its magnitude
varies. This indicates that robust generalization overfitting is not tied
to a single learning rate or training duration.

\begin{figure}[ht]
    \centering
    \includegraphics[width=0.5\linewidth]{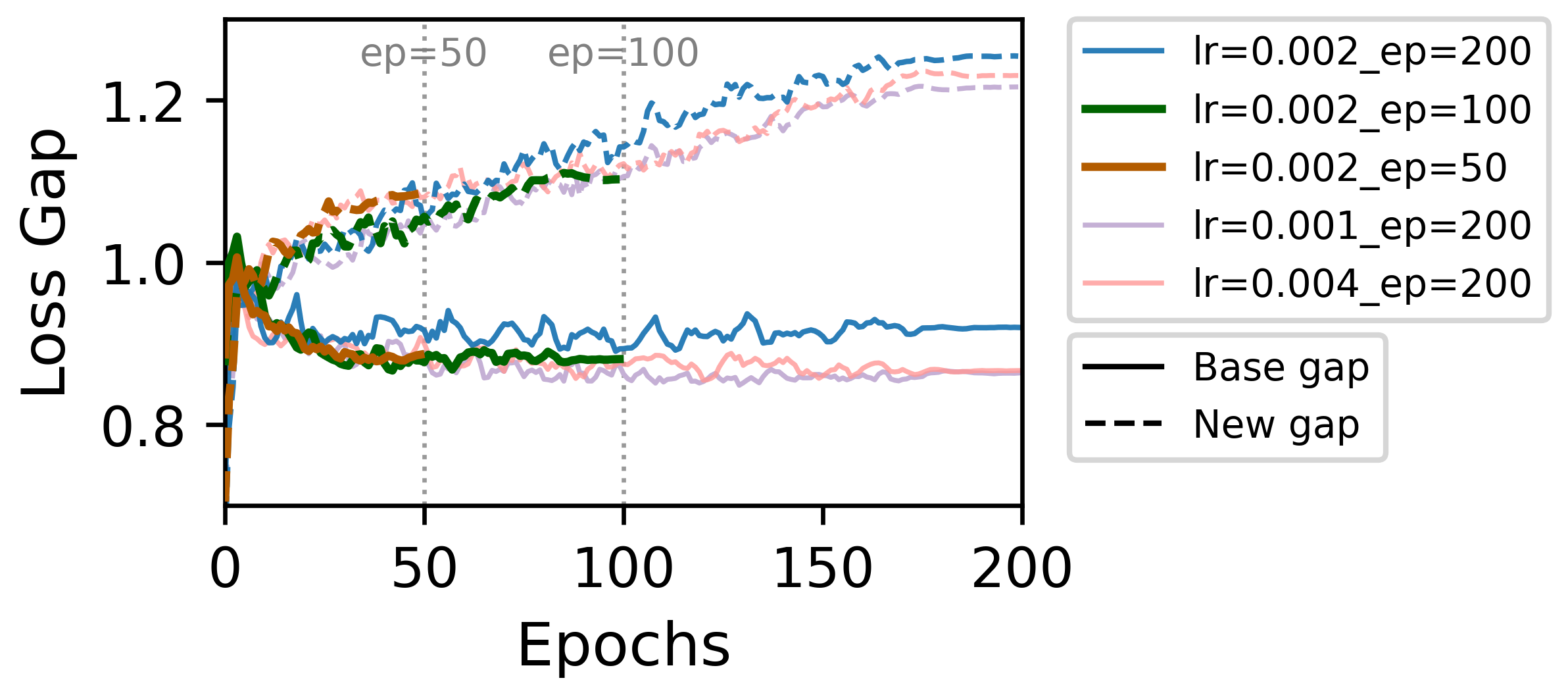}
    \caption{Robust generalization overfitting under different learning
    rates and training epochs.}
    \label{fig:rgo_settings}
    \Description{.}
\end{figure}

\subsection{Decoy-Guided Adaptive Attack}
\label{app:dgaa}
We design a Decoy-Guided Adaptive Attack (DGAA) tailored to ADAPT.
For each non-ground-truth class, DGAA measures (1) how strongly that class
competes with the ground-truth class under the target prompt and (2) how
strongly it matches the shortcut information captured by the decoy prompts.
DGAA combines these scores to select the three most shortcut-confusing
classes, runs target-prompt margin PGD toward each selected class, and keeps
the adversarial example that most strongly reduces the target-prompt margin.

\begin{table}[ht]
    \centering
    \caption{Robustness of ADAPT under the decoy-guided adaptive attack.
    Results are averaged over 11 datasets with $\epsilon=4/255$.}
    \label{tab:dgaa}
    \setlength{\tabcolsep}{5pt}
    \begin{tabular}{c|ccc|ccc}
        \toprule
        & \multicolumn{3}{c|}{Base}
        & \multicolumn{3}{c}{New} \\
        \cmidrule{2-7}
        & Acc. & PGD & DGAA & Acc. & PGD & DGAA \\
        \midrule
        ADAPT & 60.04 & 27.48 & 23.62 & 38.98 & 15.65 & 13.66 \\
        \bottomrule
    \end{tabular}
\end{table}

As shown in Table~\ref{tab:dgaa}, DGAA reduces robustness relative to
standard PGD, from 27.48\%/15.65\% to 23.62\%/13.66\% on the base/new
classes. ADAPT nevertheless retains non-trivial robustness under an attack
that explicitly uses the decoy prompts to identify shortcut-confusing
targets.




\end{document}